\documentclass[lettersize,journal]{IEEEtran}
\IEEEoverridecommandlockouts
\usepackage{dsfont}
\usepackage{algpseudocode}
\usepackage{defs}
\usepackage{algorithm,ulem}
\usepackage{cite}
\usepackage{afterpage}
\usepackage{amsmath,amssymb,amsfonts}
\usepackage{subcaption}
\usepackage{graphicx}
\usepackage{float}
\usepackage{amsmath}
\usepackage{times}
\usepackage{fancyhdr,amssymb}
\usepackage{textcomp}
\usepackage{xcolor}
\def\BibTeX{{\rm B\kern-.05em{\sc i\kern-.025em b}\kern-.08em
    T\kern-.1667em\lower.7ex\hbox{   E}\kern-.125emX}}

\usepackage[colorinlistoftodos,prependcaption]{todonotes}

\begin{document}

\title{Low-Regret and Low-Complexity Learning \\ for Hierarchical Inference
}

\author{  \IEEEauthorblockN{Sameep Chattopadhyay$^{1}$, Vinay Sutar$^{2}$, Jaya Prakash Champati$^{3}$, and Sharayu Moharir$^{2}$} \\
  \IEEEauthorblockA{$^{1}$Paul G. Allen School of Computer Science \& Engineering, University of Washington, $^{2}$Department of Electrical Engineering, IIT Bombay, $^{3}$Department of Computer Science, University of Victoria \\
                    Email:  sameepch@cs.washington.edu, vinay.sutar@iitb.ac.in, jpchampati@uvic.ca, sharayum@ee.iitb.ac.in}
}

\maketitle

\begin{abstract}

This work focuses on Hierarchical Inference (HI) in edge intelligence systems, where a compact Local-ML model on an end-device works in conjunction with a high-accuracy Remote-ML model on an edge-server. HI aims to reduce latency, improve accuracy, and lower bandwidth usage by first using the Local-ML model for inference and offloading to the Remote-ML only when the local inference is likely incorrect. A critical challenge in HI is estimating the likelihood of the local inference being incorrect, especially when data distributions and offloading costs change over time—a problem we term Hierarchical Inference Learning (HIL). We introduce a novel approach to HIL by modeling the probability of correct inference by the Local-ML as an increasing function of the model's confidence measure, a structure motivated by empirical observations but previously unexploited. We propose two policies, HI-LCB and HI-LCB-lite, based on the Upper Confidence Bound (UCB) framework. We demonstrate that both policies achieve order-optimal regret of $O(\log T)$, a significant improvement over existing HIL policies with $O(T^{2/3})$ regret guarantees. Notably, HI-LCB-lite has an $O(1)$ per-sample computational complexity, making it well-suited for deployment on devices with severe resource limitations. Simulations using real-world datasets confirm that our policies outperform existing state-of-the-art HIL methods.

\end{abstract}

\begin{IEEEkeywords}
Hierarchical Inference, Edge Computing, Online Learning, Adaptive Offloading
\end{IEEEkeywords}
\vspace*{-0.1 in}
\section{Introduction}
\vspace*{-0.1 in}


Edge-based inference offers fundamental advantages over a cloud-centric approach by enabling local data processing, which simultaneously reduces end-to-end latency and minimizes bandwidth consumption.
Significant advances in model compression and optimization have led to compact deep learning (DL) models like MobileNet \cite{sandler2018mobilenetv2}, EfficientNet \cite{tan2019efficientnet}, and Gemma-2B LLM \cite{gemmateam2024}, which run efficiently on smartphones, and tinyML models like ResNet-8 \cite{banbury2021mlperf} for micro-controllers, thus making edge-based inference possible. 
 However, due to their small size, these compact DL models have relatively
poor inference accuracy, which limits their generalization
capability and robustness to noise. 

We consider an edge intelligence system consisting of a compact DL model (Local-ML) hosted on the end-device and a state-of-the-art high accuracy DL model (Remote-ML) on an edge-server as illustrated in Figure~\ref{fig:model}.   Hierarchical Inference (HI) has emerged as an efficient distributed inference technique for such systems \cite{Nikoloska2021, alatat-hi, behera2023improved, hilf, Beytur2024,hedge-hi,behera2025exploring,datta2025online}. As shown in Figure~\ref{fig:model}, under HI, each data sample first undergoes inference on the Local-ML. Following this, a data sample is offloaded to a state-of-the-art DL model hosted on an edge/cloud server (Remote-ML model) only if the local inference is likely to be incorrect. In \cite{alatat-hi,behera2025exploring}, HI was shown to achieve a better trade-off between accuracy and offloading costs compared to the alternative distributed or collaborative inference approaches, including inference load balancing \cite{wang2018,ogden2020,Fresa2023} and DNN partitioning  \cite{kang2017,Li2020,hu2022}.

\begin{figure}[t]
    \centering
    \includegraphics[scale=0.25]{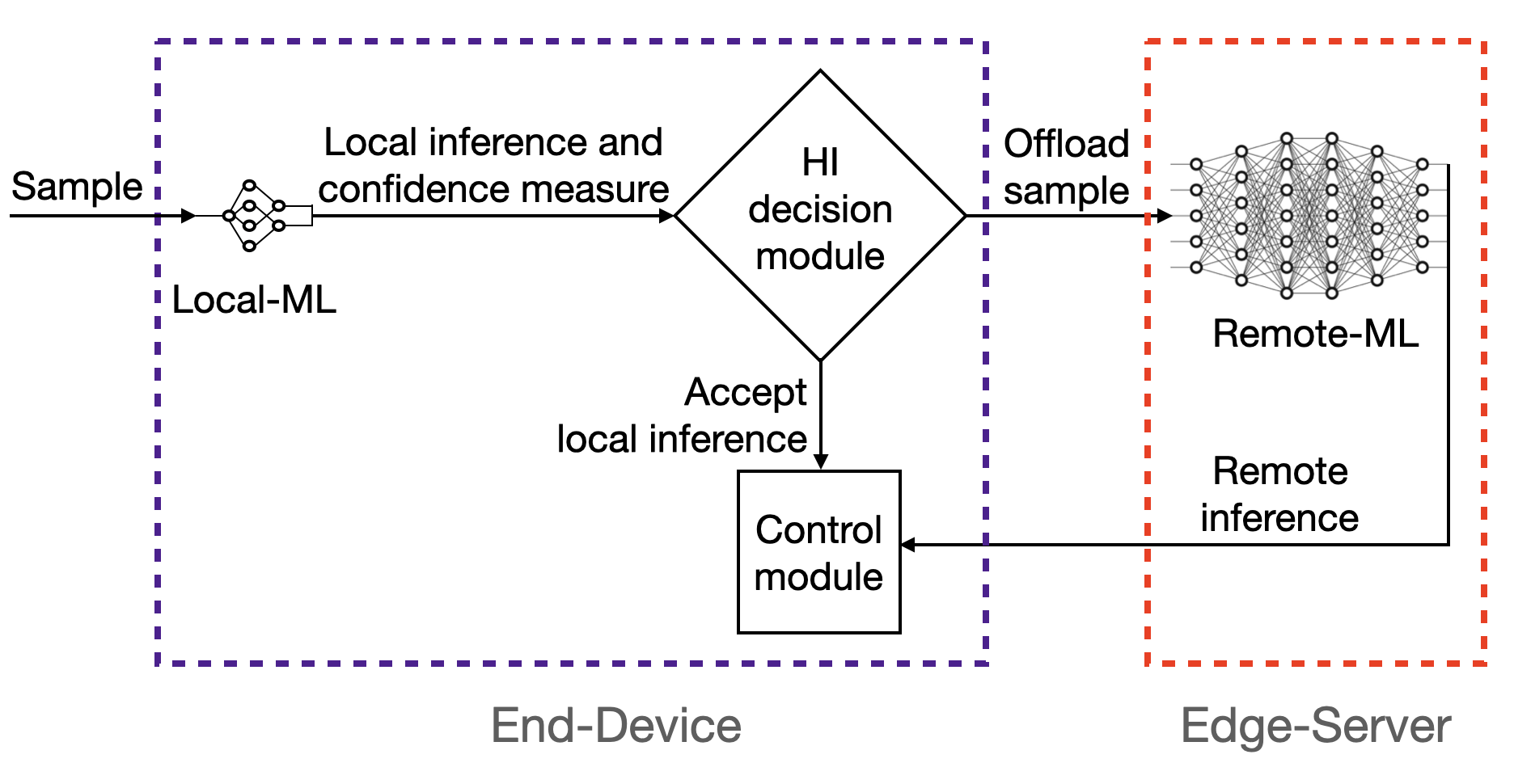}
    \vspace*{-0.1 in}
    \caption{Hierarchical Inference on an edge intelligence system comprising a compact on-device Local-ML model and an edge-server with a larger, more accurate Remote-ML model.}
    \label{fig:model}
    \vspace*{-0.3 in}
    \end{figure}

Implementing HI systems critically depends on assessing the likelihood of incorrect inference by the Local-ML model. Much of the prior research in HI has concentrated on multi-class classification \cite{hilf,hedge-hi,Beytur2024}. The Local-ML model trained for this task outputs a predicted class for each sample based on a confidence measure. 
When the Local-ML model is a neural network, the typical choice for the confidence measure is the largest output of the last softmax layer. However, our policies and analysis work for any confidence measure. A sample is offloaded if and only if the confidence measure of the Local-ML model for that sample is below a (time-varying) threshold. Initial approaches \cite{Nikoloska2021, alatat-hi, behera2023improved} computed an optimal, fixed threshold offline using labeled training data under a fixed offloading cost. However, real-world inference data often diverges from training data, and offloading costs can be time-varying. To address this, \cite{hilf} formulated the HI Learning (HIL) problem to dynamically learn the best offloading threshold. 

Existing HIL methods \cite{hilf,hedge-hi,Beytur2024} operate without assuming any underlying structure in the relationship between a model's  confidence measure and the probability of incorrect inference, i.e., misclassification. These methods rely on exponential-weight algorithms that have $O(T^{2/3})$ regret bounds and a per-sample computational complexity of $O(N)$, where $N$, the number of distinct confidence measure values, can be  large. This paper investigates an open question: \textit{Can we develop HIL policies with lower computational complexity and improved regret guarantees by assuming a well-motivated structural relationship between the model's confidence measure and its likelihood of misclassification?}

\vspace*{-0.5 em}
\subsection{Our Contributions}
\vspace*{-0.5 em}
\subsubsection*{Modeling Contribution}
We model the probability of a correct inference by the Local-ML model for a sample as an (unknown) increasing function of the confidence measure output by the Local-ML model for that sample. This is motivated by our empirical studies on various tiny-ML models applied to multiple real-world datasets. This is a key departure from the existing literature on HIL.



\subsubsection*{Algorithmic and Analytical Contributions}  While our problem shares similarities with the contextual multi-arm bandit setting \cite{li2010contextual}, it diverges significantly due to its unique feedback structure. This distinction introduces challenges that require novel policies and a dedicated theoretical analysis. \color{black}

We propose two policies for the HIL problem based on the Upper Confidence Bound (UCB) framework \cite{ucb}, called HI using Lower Confidence Bound (HI-LCB) \color{black} and HI-LCB-lite. 
\begin{enumerate}[itemsep=0.5em]
\item[--] Both HI-LCB and HI-LCB-lite achieve $O(\log T)$ regret, where $T$ is the number of samples. This guarantee holds for both i.i.d. and adversarial arrivals, significantly improving upon the $O(T^{2/3})$ regret bounds of existing HIL policies \cite{alatat-hi,hilf,hedge-hi}. 
\item[--] We prove an $\Omega(\log T)$ lower bound on the regret of any online policy. Thus, HI-LCB and HI-LCB-lite have order-optimal regret with respect to the number of samples. 
\item[--] HI-LCB has space and computational complexity comparable to existing state-of-the-art policies, scaling linearly with the number of distinct confidence measure values. HI-LCB-lite has $O(1)$ computational complexity per sample, making it uniquely suited for deployment on end-devices with severe resource constraints.
\item[--] Through simulations on real-world datasets, we demonstrate that HI-LCB and HI-LCB-lite outperform existing HI policies for most parameter values.
\end{enumerate}

\subsection{Related Work}

Multiple distributed/collaborative inference approaches, including inference load balancing  \cite{wang2018,ogden2020,Fresa2023} and DNN partitioning \cite{kang2017,Li2020,hu2022} have been studied in recent years. Unlike DNN partitioning, HI does selective inference offloading, thereby enhancing the system’s responsiveness with minimal accuracy trade-off. Also, in contrast to inference load balancing \cite{wang2018,ogden2020,Fresa2023}, it does not naively balance the task load between the end-device and the edge-server.\color{black}

Early exit \cite{teerapittayanon2016branchynet} is a well-known technique used for reducing the inference latency by allowing a DL model to output predictions at intermediate
layers. Since HI operates with any off-the-shelf Local-ML and Remote-ML models, using early exit in either of these models complements HI, with the added benefit of reduced inference latency \cite{behera2025exploring}.

\color{black}

Under HI, all samples are first sent to the Local-ML model for inference. When the inference task is multi-class classification, the Local-ML model outputs a predicted class and a confidence measure for each sample. The offloading decision is made based on this confidence measure. For each sample, if the device accepts the local inference and it is correct, no loss occurs; otherwise, it incurs a loss of $1$ unit. If the device rejects the inference and offloads the sample to the Remote-ML model, it incurs an offloading loss. The existing HI offloading policies are threshold-based policies. In each round, the policy fixes a threshold, offloading samples only if the confidence measure of the Local-ML falls below it. The goal is to find the optimal value of the threshold to minimize the total loss incurred over $T$ samples. 

In \cite{hilf, hedge-hi}, this problem is formulated as a variant of the classical Prediction with Expert Advice (PEA) \cite{cesa2006book} with thresholds as experts. An online policy was proposed in \cite{hilf}, with a regret bound of $O(T^{2/3} (\log 1/\lambda_\text{min})^{1/3})$, where $\lambda_\text{min}$ is the minimum difference between any two distinct confidence measures observed for the $T$ samples. Noting that $\lambda_\text{min}$ is dataset dependent, the authors in \cite{hedge-hi} proposed Hedge-HI with $O(T^{2/3}N^{1/3})$ regret, where $N$ is an upper bound on the number of distinct confidence measure values. Furthermore, the proposed policies have $O(t)$ runtime per sample, where $t$ is the sample number. This becomes prohibitively large as $t$ increases, particularly for resource-constrained devices. 

In \cite{Beytur2024}, the authors studied HI in a system with multiple devices and explored maximizing the average accuracy under an average offloading cost constraint. This work provided a regret bound that depends on the inverse of the minimum value of the offloading cost.

In the absence of any correlation between the confidence measure and the likelihood of incorrect inference, the regret of any online HI offloading policy is $\Omega(\sqrt{T \log N})$ \cite{datta2025online} \color{black}.

Our problem shares similarities with the widely studied multi-armed bandit (MAB) framework \cite{lattimore2020bandit}, particularly the contextual bandit variant \cite{li2010contextual}. However, significant differences in the feedback model distinguish our setting from standard contextual bandits. Specifically, in the standard MAB setting, the learner observes the reward only for the pulled arm. In contrast, our setting exhibits a distinct structure: pulling a certain arm reveals the rewards of all arms, while pulling others reveals nothing. This key distinction necessitates new algorithm design and novel analysis techniques. We elaborate on these nuances in Sections \ref{sec:Terminology}, \ref{sec:results}, and \ref{sec:outlines}, following our formal problem definition.
\color{black}

\vspace*{-0.5 em}
\section{Setting}
\vspace*{-0.5 em}
\label{sec:Terminology}
\begin{figure*}[t]
    \centering
    \includegraphics[width=\textwidth]{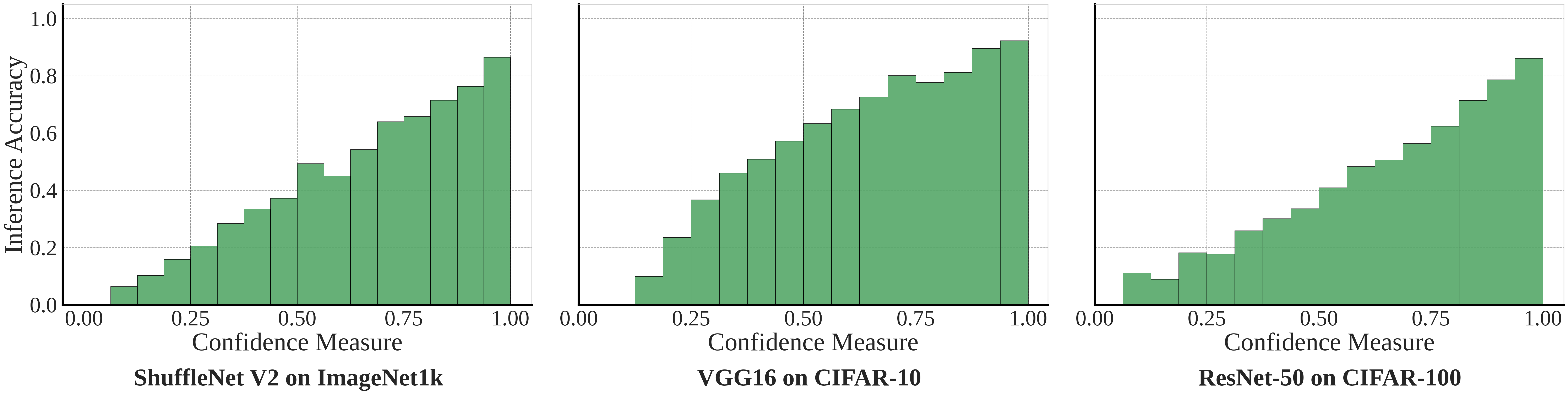}
    \caption{Empirical relationship between confidence measure (max softmax) and inference accuracy for multi-class classification. Accuracy increases with the value of the confidence measure, with rare exceptions.}
    \label{fig:model_accuracy}
    \vspace*{-0.2 in}
\end{figure*}

We consider an edge intelligence system consisting of an end-device and an edge-server, as illustrated in Figure~\ref{fig:model}. The end-device periodically collects data samples and performs a multi-class classification inference task on each sample. The edge intelligence system is equipped with two deep learning (DL) models for this task: a light-weight \textit{Local-ML} model deployed on the end-device and a larger and more accurate \textit{Remote-ML} model hosted on an edge-server.

Each data sample is first sent to the Local-ML for inference. The Local-ML model outputs a predicted class and a confidence measure for each sample. The end-device can either accept its prediction or decide to offload the sample to the larger model on the edge-server. This offload decision is made based on the confidence measure of the Local-ML model. If offloaded, the system uses the inference of the Remote-ML model hosted on the edge-server.
\vspace*{-0.5 em}
 \subsection{Inference Task}
 \vspace*{-0.5 em}
\subsubsection*{Multi-Class Classification}
Let $X$ represent the feature space from which the data samples are drawn, and $Y = \{1, \dots, m\}$ denotes the label set with $m$ classes. A classifier $h: X \to Y$ performs inference by computing $h(x)$ for a data sample $x$. Given a sample-label pair $(x, y)$, the classification error is defined as 
\begin{equation}
    \label{eq:error}
\eta(h,x)=\mathbb{1}{\{h(x) \neq y\}}.
\end{equation}
Consider a time-slotted system and let $x_t \in X$ denote the data sample that arrives at the system at time $t$. For each input sample, any DL model outputs scores $s_i(x_t)$ for each class $ i \in Y$. 
 The Local-ML model $h_l$ assigns the sample \( x_t \) to the class corresponding to the maximum score, i.e., the Local-ML model's inference is denoted by \[
     h_l(x_t) = \arg\max_{i \in Y} s_i(x_t). 
\]
 Similarly, when $x_t$ is offloaded to the server, the Remote-ML model $h_r$ returns a corresponding prediction. 
  
\subsubsection*{Confidence Measure}
A confidence measure quantifies a model's confidence in its prediction for a given data sample \cite{gawlikowski2023survey}. In general, the confidence measure for a sample $x_t$  is a deterministic function $g$ of the output scores $s(x_t)=[s_1(x_t), s_2(x_t), \dots, s_m(x_t)]$, and is defined as $g: s(x_t) \mapsto \Phi$, where $\Phi \subseteq [0, 1]$. 
In this work, we consider that the confidence measures from the Local-ML model are quantized into a finite set 
$
   \Phi=\{\phi_1,\phi_2,\dots \phi_K\}.
$
Without loss of generality, the confidence measure values are indexed in the increasing order, i.e. $0<\phi_1\leq\phi_2\leq\dots \phi_K\leq1$.

\begin{remark}
    Note that the output of the DL models on devices are quantized with the weights and the score values stored typically using in 4, 8, or 16 bits \cite{hubara2018quantized}.  For example, an 8-bit quantized ResNet8 \cite{banbury2021mlperf} has at most $256$ unique maximum soft-max values. Thus, $\Phi$ is finite in practice.  
\end{remark}
\color{black}
Our analytical results hold for any confidence measure, while in the numerical results presented in this work, we use the widely used maximum softmax value \cite{hendrycks2017} as the confidence measure. Let $\phi(t) \in \Phi$ denote the confidence measure  for the sample $x_t$, the max-softmax confidence is given by  
\[
\phi(t) = g(s(x_t))=  \max_i  \frac{\exp(s_i(x_t))}{\Sigma_{j=1}^{m} \exp(s_j(x_t))}.
\]
Recall that during classification, the sample \( x_t \) is assigned to the class corresponding to the highest score; thus, in this case, the confidence measure for a given sample is the softmax value corresponding to the predicted class. 

The relationship between confidence measures and the accuracy of DL-based models has been extensively studied in the literature on uncertainty estimation for deep learning models \cite{Seo2018LearningFS,mehrtash2019confidence,10.5555/3495724.3497255,gawlikowski2023survey}. However, prior work on hierarchical inference has not modeled this structure and therefore has not exploited it to make offloading decisions. 
We model the relationship between the Local-ML model's confidence measure and misclassification, and consequently its inference accuracy, as follows:
\begin{equation}
    \label{eq:struc}
    \Pr\left\{h_l(x_t) =y_t \mid \phi(t)\right\}= f(\phi(t)),
\end{equation}
 independent of all other samples.  Here $f(\cdot)$ is a \textit{non-decreasing} function bounded in $[0,1]$, representing the probability of accurate inference. 


 
Our structural assumption linking confidence measures to prediction accuracy  (given in \eqref{eq:struc}) is strongly supported by empirical evidence. A subset of the results of our empirical studies is presented in Figure \ref{fig:model_accuracy}. As shown in the figure, the inference accuracy $f(\cdot)$ steadily increases with the max-softmax confidence measure with rare exceptions. This trend is observed across both VGG16 \cite{vgg16} and ResNet-50  \cite{he2016deep} models, which serve as widely used baselines in computer vision, evaluated on the CIFAR-10 and CIFAR-100 datasets \cite{krizhevsky2009learning}, respectively. In addition, we also verified this trend for an extremely compact model, ShuffleNet V2 \cite{ma2018shufflenet}, which is commonly deployed on resource-constrained devices. The ShuffleNet V2 model was evaluated on the ImageNet1k dataset \cite{deng2009imagenet}, demonstrating similar patterns for the inference accuracy. 

\subsection{Algorithmic Challenge and Cost Model} At each time $t$, the system receives the Local-ML model’s inference for sample $x_t$, comprising a predicted class $h_l(x_t)$ and a confidence measure $\phi(t)$. Based on the confidence measure, the system chooses to either accept the local inference or offload the sample to the Remote-ML model.
The algorithmic task is to determine when to offload and when to accept. At time $t$, a policy $\pi$, outputs a decision $D_\pi(t)$ given by:
\[
\mathnormal{D_\pi(t)} =
\begin{cases}
0, & \text{Accept the Local-ML model's inference} \\
1, & \text{Offload task to the Remote-ML model}.
\end{cases}
\]


 A critical challenge in such systems is the unavailability of ground truth during deployment. We address this by using the output of the Remote-ML as a proxy for the ground truth, i.e., $y_t = h_r(x_t)$. Given that the Remote-ML model is a state-of-the-art, high-accuracy system, it serves as the definitive baseline in the absence of labeled data. This is a standard approach also adopted in \cite{hedge-hi}. \color{black}

When the system accepts the Local-ML model’s inference, it incurs a unit cost for each incorrect inference and zero cost if the inference is correct. 
For each sample that is offloaded to the Remote-ML, the system incurs an offloading corresponding to the communication overhead and/or latency. Let $\Gamma_t\in [0,1]$ denote the offloading cost at time $t$. We model  $\Gamma_t$ as an i.i.d. random variable with an unknown distribution with unknown mean $\gamma$. 
Let $L_t^{\pi}$ denote the cost incurred by the system at time $t$ under policy $\pi$. It follows that 
\begin{equation}
\label{eq:cost}
    L_t^{\pi} = \Gamma_t \cdot D_\pi(t) + \eta(h_l,x_t) \cdot (1-D_\pi(t)).
\end{equation}
We consider the setting where the function $f(\cdot)$ is unknown and the offloading cost is also a priori unknown to the system. Both the offloading cost and the inference of the Remote-ML model are revealed only when a sample is offloaded. Consequently, offloading is the sole method to estimate this cost and determine if the Local-ML model's inference matches the Remote-ML model's, which is essential for estimating $f(\cdot)$.

\begin{remark}
\label{rem:fc_defn}
We also consider the setting where $\Gamma_t$ is fixed and known as a separate special case. 
\end{remark}

While our analytical results and a subset of simulations use the output of the Remote-ML as a proxy for the ground truth, in Appendix B, we present simulation results where the actual ground truth serves as the baseline, explicitly accounting for potential Remote-ML inference errors while computing the cost incurred by the system. For this subset of simulation results, we use an appropriately modified cost function, which we define in Appendix B.
\color{black}

\subsection{Arrival Processes}
We consider a time-slotted system where the end-device collects one data sample per time-slot. We study two settings differentiated by the assumptions made on the arrival process of the samples and the corresponding confidence measure values. 

\subsubsection*{Adversarial Arrivals} In the adversarial setting, we make no structural assumptions on the arrival sequence of the confidence measure values for the local model's inferences, which is represented by $\sigma_T=\{\phi(t)\}_{1\leq t\leq T} \in \Phi^T$.

\subsubsection*{Stochastic Arrivals} In the stochastic setting, we make the following
assumption.

\begin{assumption} 
\label{as:stoch}
The confidence measure value for a Local-ML model's inference is independent and identically distributed across time. More specifically, for all $t \in \{1, 2, \dots, T\}$ and $i \in \{1, 2, \dots, |\Phi|\}$,
\[
    \Pr(\phi(t)=\phi_i)=w_i.
\]
\end{assumption} 
\begin{remark}
During the training of deep learning models, it is typically assumed that the data samples are drawn i.i.d.  from an unknown distribution over $X$ \cite{Shen2021TowardsOG}. Under the stochastic setting (Assumption \ref{as:stoch}), we extend this assumption to the inference phase, where samples arrive in an i.i.d. The i.i.d. assumption \cite{Beytur2024,hedge-hi} during inference is well-motivated in applications such as image classification or speech recognition systems, where user queries arrive independently. \color{black}  \end{remark}

\subsection{Performance Metrics}
\label{sec:perf}
\color{blue}
\color{black}

 We focus on a class of policies called threshold-based policies, studied in \cite{hilf, alatat-hi, hedge-hi}. A policy in this class maintains a scalar threshold value that can evolve over time. Once a sample is processed by the Local-ML model, the policy offloads the sample if and only if the Local-ML model's confidence measure for that sample is below the threshold. We compare the performance of our online threshold-based HI policy with the optimal static threshold policy, denoted by $\pi^*$, characterized in Section \ref{sec:policies}. The optimal static threshold policy has also been used as the benchmark in \cite{hilf,alatat-hi, hedge-hi}.

Under adversarial arrivals, the regret of a policy $\pi$, denoted by $R_\pi^A(T)$, is
defined as the maximum difference in the expected cumulative costs incurred by the policy $\pi$ and $\pi^\ast$ under any arrival sequence $\sigma_T \in \Phi^T$. Thus, 
\[
    R_\pi^A(T) = \max_{\sigma_T} \;\mathds{E}
    \left[\sum_{t=1}^T
L_t^{\pi} - \sum_{t=1}^T L_t^{\pi^{\ast}}\right],
\]
where the expectation is over the randomness in the policy $\pi$, the predictions of the Local-ML model being correct, and the offloading costs.

For stochastic arrivals, i.e., under Assumption \ref{as:stoch}, the regret for policy $\pi$, denoted by $ R_\pi^S(T)$, is defined as the difference in the expected cumulative costs incurred by the policies $\pi$ and $\pi^\ast$. Thus,
\[
    R_\pi^S(T)=\mathds{E}
    \left[\sum_{t=1}^T
L_t^{\pi} - \sum_{t=1}^T L_t^{\pi^{\ast}}\right],
\]
where the expectation is over the randomness in the policy $\pi$, the predictions of the Local-ML model being correct, the arrival sequence $\sigma_T$, and the offloading costs.

\begin{remark}
\label{rem:contextual}
 Our setting can be viewed as a variation of the contextual bandit problem \cite{li2010contextual}, using confidence measure ($\phi_t$) as the context and offload/accept as the arms. The fundamental difference is that standard contextual bandits reveal the pulled arm's reward in every round, whereas our setting reveals both arm rewards only upon offloading, and no information when accepting Local-ML inference. 
\end{remark}

\section{Our Policies: HI-LCB and HI-LCB-lite}
\label{sec:policies}
In this section, we propose two novel offloading policies. Consistent with real-world online settings, these policies do not have access to the true function $f(\cdot)$ and the average offloading cost $\gamma$; instead, the policies gradually learn these by selectively offloading samples to the server. Since offloading is the only means of obtaining feedback about inference accuracy and costs, it effectively serves as the mechanism for \textit{exploration}. 

Before discussing our proposed policies, we characterize the optimal static threshold policy. This policy, denoted by $\pi^\ast$, uses a fixed threshold value that minimizes the system's expected cost. Our policies aim to mimic this optimal static threshold policy while accounting for the uncertainty in $f(\cdot)$ and the expected offloading cost $\gamma$.

\begin{definition}
    We partition $\Phi$ into two sets: 
    \begin{align*}
        \Phi_H = \left\{ \phi_i \in \Phi \, : \, 1-f(\phi_i) < \gamma \right\} \text{ and }\Phi_L = \Phi \setminus \Phi_H.
        \end{align*}
\end{definition}
\vspace*{-0.1 in}
\begin{lemma}
\label{lem:opt}
    The policy $\pi^\ast$ makes offloading decisions as follows:
\[
\begin{aligned}
\label{decision}
 \mathnormal{D_{\pi^\ast}(t)}=\begin{cases} 0\;\text{(Accept)} 
      & \text{if $\phi(t) \in \Phi_H$},\\
    1\;\text{(Offload)}  & \text{otherwise}.

  \end{cases}
  \end{aligned}
\]
\end{lemma}
\vspace*{-0.1 in}
\begin{proof}
Under any policy $\pi$, given $\phi(t)$, the expected cost incurred on accepting the Local-ML model's inference is 
\begin{align*}
\mathds{E}[L_t^{\pi}\,|\,D_{\pi}(t)=0, \phi(t)] &=\mathds{E}[\eta(h_l,x_t)|\phi(t)]=1-f(\phi(t)),
\end{align*}
and the expected cost incurred on offloading the sample is 
\begin{align*}
\mathds{E}[L_t^{\pi}\,|\,D_{\pi}(t)=1, \phi(t)] &= \mathds{E}[\Gamma_t]= \gamma.   
\end{align*}
To minimize the expected cost, the Local-ML model's inference should be accepted in slot $t$ if and only if 
\begin{align}
    \mathds{E}[L_t^{\pi}\,|\,D_{\pi}(t)=0, \phi(t)]&<\mathds{E}[L_t^{\pi}\,|\,D_{\pi}(t)=1, \phi(t)] \nonumber \\
    \iff 1-f(\phi(t)) &< \gamma. 
\label{eq:optrule}
\end{align}
It follows that the cost incurred is minimized by offloading samples with confidence measures in $\Phi_L$ and accepting the local inference of samples with confidence measures in $\Phi_H$. Since $f(\cdot)$ is a non-decreasing function, it follows that for all $\phi_i \in \Phi_L$ and $\phi_j \in \Phi_H$, $\phi_i < \phi_j $. The policy that accepts the Local-ML model's inference if $\phi(t) \in \Phi_H$ and offloads otherwise is thus the optimal static threshold policy $\pi^{\ast}$.
\end{proof}

\subsection{HI-LCB}

Our first policy called HI-LCB (Algorithm \ref{alg:hilcb}), is a threshold-based online offloading strategy. \textit{HI-LCB leverages the non-decreasing relationship between confidence measures and expected inference accuracy to make offloading decisions}. 

Initially, the policy offloads all incoming samples to gather sufficient information on both accuracy and offloading costs. This exploration continues until the policy can, with high certainty, determine that a specific value of the confidence measure belongs to the high-confidence set $\Phi_H$. \color{black} The policy exploits the fact that $f(\cdot)$ is increasing to make this determination. \color{black} Once this determination is made, the policy stops offloading samples with that confidence measure and continues to accept local inferences until it becomes uncertain about the correctness of its decision. Conversely, for the remaining values of the confidence measure, the policy continues to offload all samples to the Remote-ML model on the server.

For each $\phi_i \in \Phi$, HI-LCB keeps track of the following quantities over time-slots indexed by $t$:
\begin{enumerate}
    \item[--] Number of offloads corresponding to $\phi_i$, given by $\mathsf{O}_{\phi_i}$.
    \item[--] An empirical estimator of $f(\phi_i)$ denoted by $\hat{f}(\phi_i)$.
    \item[--] The LCB of the $f(\phi_i)$, denoted by $\text{LCB}_{\phi_i}$, where:
\end{enumerate}
\begin{align}
\label{eq:LCB_def}
\text{LCB}_{\phi_i} &= \max_{\phi_j\leq \phi_i} \left\{ \hat{f}(\phi_j) - \sqrt{\frac{\alpha \log t}{\mathsf{O}_{\phi_j}}} \right\}. 
\end{align}

\begin{remark}
 The LCB definition in \eqref{eq:LCB_def} leverages the non-decreasing nature of $f(\cdot)$ by forcing $\text{LCB}_{\phi_i} \geq \text{LCB}_{\phi_j} $ for all $\phi_j \leq \phi_i$.
\end{remark} \color{black}

HI-LCB also maintains an empirical estimate and an LCB of the offloading cost, denoted by $\hat{\gamma}$ and $\text{LCB}_\gamma$ respectively, where $\text{LCB}_\gamma$ is defined as follows:
 \begin{equation}
 \label{eq:LCB2}
     \text{LCB}_\gamma=\hat{\gamma} - \sqrt{\frac{\alpha \log t}{\mathsf{O}_{\gamma}}},
 \end{equation}
 where $\mathsf{O}_{\gamma} = \sum_{\phi_i\in\Phi}\mathsf{O}_{\phi_i}$ is the total number of offloads across all confidence measures in $\Phi$.\color{black}

To mimic the optimal static threshold policy while accounting for the uncertainty in $f(\cdot)$ and $\gamma$, HI-LCB makes the offloading decision in the following manner:
\[
\mathnormal{D_\pi(t)}=\begin{cases}
    0 & \text{if $1 - \text{LCB}_{\phi(t)} < \text{LCB}_\gamma$},\\
    1 & \text{otherwise}.
  \end{cases}
\]
\begin{algorithm}[t]
\caption{HI-LCB }
\begin{algorithmic}[1]
\Require Data samples $\{x_t\}_{1 \leq t\leq T}$ and confidence set \( \Phi \) \color{black}
\Ensure Offload decisions $\{D_\pi(t)\}_{1 \leq t\leq T}$

\State \textbf{Parameters:} Exploration parameter \( \alpha \)

\State \textbf{Initialize:} $\hat{\gamma}=0$, \(\mathsf{O}_{\gamma} = 0 \), \(\mathsf{O}_{\phi_i} = 0 \), \( \hat{f}(\phi_i) = 0, \) $\forall  \phi_i \in \Phi $.
\For{$t = 1, 2, \dots$}
    \State Receive inference $h_l(x_t)$ and score $\phi(t) \in \Phi$    
    \State Compute
    $
    \text{LCB}_{\phi(t)} $ using \eqref{eq:LCB_def}
 and $\text{LCB}_\gamma$ using \eqref{eq:LCB2}.
    \If{$1 - \text{LCB}_{\phi(t)} \geq \text{LCB}_\gamma$ or $\mathsf{O}_{\phi(t)} = 0$}
        \State $D_\pi(t) = 1$ \hfill // Offload and obtain $y_t$ and $\Gamma_t$ 
         \State $\hat{f}(\phi(t)) \leftarrow \frac{\mathsf{O}_{\phi(t)} \cdot \hat{f}(\phi(t)) + 1-\eta(h_l,x_t)}{\mathsf{O}_{\phi(t)} + 1}$ 
        \State $\hat{\gamma} \leftarrow \frac{\mathsf{O}_\gamma \cdot \hat{\gamma} + \Gamma_t}{\mathsf{O}_\gamma + 1} $ 
        \State $\mathsf{O}_{\phi(t)} \leftarrow \mathsf{O}_{\phi(t)}+ 1$, $\mathsf{O}_{\gamma} \leftarrow \mathsf{O}_{\gamma}+ 1$
    \Else
        \State $D_\pi(t) = 0$ \hfill // Accept local model inference

    \EndIf
\EndFor
\end{algorithmic}
\label{alg:hilcb}
\end{algorithm}
\begin{remark}
    The HI-LCB policy requires $O(|\Phi|)$ operations per sample and has $O(|\Phi|)$ space complexity. 
\end{remark}
\subsection{HI-LCB-lite}
A computational complexity of $O(|\Phi|)$ operations per sample may be prohibitively large for resource-constrained end-devices. For such situations, we propose our second policy called HI-LCB-lite. HI-LCB-lite follows the same steps as HI-LCB except that the LCB definition in \eqref{eq:LCB_def} is replaced with 
    \begin{align}
    \label{eq:vanillaLCB}
      \text{LCB}_{\phi_i} =    \hat{f}(\phi_i) - \sqrt{\frac{\alpha \log t}{\mathsf{O}_{\phi_i}}}. 
    \end{align}

\begin{remark}
    Unlike HI-LCB, HI-LCB-lite does not exploit the non-decreasing nature of $f(\cdot)$ in the LCB definition and therefore, requires $O(1)$ operations per sample, making it well-suited for deployment on resource-constrained devices. 
\end{remark}  
\color{black}
\begin{remark}
While HI-LCB is inherently a threshold-based policy due to the max operator in its definition of $\text{LCB}_{\phi_i}$ in \eqref{eq:LCB_def}, HI-LCB-lite may not always exhibit threshold-based offloading behavior. 
\end{remark}
\begin{remark}[Fixed Offloading Cost]
\label{rem:fc}
    For the special case where the offloading cost is fixed and a priori known, i.e., $\Gamma_t = \gamma$, $\forall \,t$, the $\text{LCB}_{\gamma}$ definition in \eqref{eq:LCB2} is replaced with $\text{LCB}_{\gamma} = \gamma$ for both HI-LCB and HI-LCB-lite.
\end{remark}

\color{black}
     \section{Analytical Results and Discussion}
     \label{sec:results}
In this section, we present our analytical results. The proof outlines of our results are presented in Section \ref{sec:outlines}, with detailed proofs provided in Appendix A. 

While our setting shares similarities with the contextual multi-armed bandit (MAB) framework, our distinct feedback structure introduces a critical challenge. In the classical UCB framework for MABs, the learning process is naturally self-correcting: if noisy samples initially cause a sub-optimal arm to appear optimal, the learner pulls that arm, generates new data, and corrects the estimate. This eventually reveals the arm's sub-optimality and reduces regret.

In contrast, our setting lacks this self-correcting mechanism due to the nature of the feedback. Crucially, if the learner accepts the Local-ML inference, no further feedback is observed. Therefore, if initial noise leads to an incorrect acceptance for a specific confidence value, the learner never receives the data required to correct the error. This can cause the learner to get permanently stuck on the wrong arm, leading to linear regret. This fundamental difference necessitates an independent analysis to upper bound the probability of such `bad events,' a step not required in standard contextual MAB analysis. 
\color{black}

To simplify the following discussion, we define:
\begin{itemize}[itemsep=0.5em]
    \item[--] $\Delta_{\phi_i}=|1-f(\phi_i)-\gamma|$,
    \item[--] $\Delta = [\Delta_{\phi_1},\Delta_{\phi_2}\dots\Delta_{\phi_{|\Phi|}}]$,
    \item [--]$\Phi_H^{(i)}= \left\{ \phi_j \in \Phi_H \mid \phi_j \leq \phi_i \right\}$,
     \item[--] { $\displaystyle \mathbf{C_1}(\Delta)  = \sum_{\phi_i \in \Phi_H} \;\left( \frac{4\alpha \Delta_{\phi_i}}{2\alpha-1}  \right)+ \max_{ \Phi_L}\; 2\alpha\Delta_{\phi_j} \left(\frac{|\Phi_L|+1 }{{2\alpha-1}}\right)$},
    \item[--] {$ \displaystyle\mathbf{C_2}(\Delta) = \frac{2\alpha }{2\alpha-1} \left(\sum_{\phi_i \in\Phi_H} \;\Delta_{\phi_i} + |\Phi_L|\max_{ \Phi_L}\,\Delta_{\phi_j} \right)$},
    \item[--] {\small $\displaystyle \mathbf{C_3}(\Delta)=\sum_{ \phi_i \in \Phi_H}\frac{4\alpha\Delta_{\phi_i} }{2\alpha-1}  \min_{\Phi^{(i)}_H}\left(   \frac{ w_i}{w_j} \right)+\max_{ \Phi_L}\; 2\alpha\Delta_{\phi_j} \left(\frac{|\Phi_L|+1 }{{2\alpha-1}}\right),$}
    \item[--] {$ \displaystyle\mathbf{C_4}(\Delta) = \frac{2\alpha }{2\alpha-1} \left(\sum_{\phi_i \in\Phi_H} \;\min_{\Phi^{(i)}_H}\left(   \frac{ w_i\Delta_{\phi_i}}{w_j} \right) + |\Phi_L|\max_{ \Phi_L}\,\Delta_{\phi_j} \right)$}.
\end{itemize}
\color{black}
Our first theorem provides regret guarantees for HI-LCB and HI-LCB-lite for adversarial arrivals. 
\begin{theorem} [Adversarial Arrivals] 
    \label{thm:adv}
   Under adversarial arrivals and i.i.d. offloading costs, for any $\alpha> 0.5$:
\begin{enumerate}
    \item[(a)] $\displaystyle R_{\text{HI-LCB}}^A(T) \leq \sum_{\phi_i \in \Phi_H} \; \frac{16\alpha}{\Delta_{\phi_i}}\log T\; +\mathbf{C_1}(\Delta)$, 
    \item[(b)] $\displaystyle R_{\text{HI-LCB-lite}}^A(T) \leq \sum_{\phi_i \in \Phi_H} \; \frac{16\alpha}{\Delta_{\phi_i}}\log T\; +\mathbf{C_1}(\Delta)$.
\end{enumerate}
Further, for the special case with fixed and a priori known offloading cost discussed in Remark \ref{rem:fc_defn} and Remark \ref{rem:fc},
\begin{enumerate}
    \item[(c)] $\displaystyle R_{\text{HI-LCB}}^A(T) \leq \sum_{\phi_i \in \Phi_H} \; \frac{4\alpha}{\Delta_{\phi_i}}\log T\; +\mathbf{C_2}(\Delta)$, 
    \item[(d)] $\displaystyle R_{\text{HI-LCB-lite}}^A(T) \leq \sum_{\phi_i \in \Phi_H} \; \frac{4\alpha}{\Delta_{\phi_i}}\log T\; +\mathbf{C_2}(\Delta)$.
\end{enumerate}
\end{theorem}
Theorem \ref{thm:adv} establishes the first $O(\log T)$ regret bound for the HIL problem. This marks a significant improvement over the current state-of-the-art $O(T^{2/3})$ guarantees achieved by {HIL-F}\cite{hilf} and {Hedge-HI}\cite{hedge-hi}. We achieve this by modeling a direct relationship between the Local-ML model's confidence measures and its probability of misclassification, which the HI-LCB and HI-LCB-lite policies leverage to make offloading decisions. Note that for adversarial arrivals, our regret bounds for HI-LCB and HI-LCB-lite are identical. Furthermore, the guarantees for known fixed offloading costs are tighter than those for i.i.d. offloading costs with an unknown mean.
\begin{remark}
\label{rem:info}
Existing HI offloading policies, such as {HIL-F} \cite{hilf} and {Hedge-HI} \cite{hedge-hi}, assume that offloading costs are revealed to the system even when it chooses to accept the local inference. HI-LCB and HI-LCB-lite, in contrast, operate under a more limited information structure, i.e., the offloading cost is revealed only if the system decides to offload, and yet have much stronger performance guarantees.
\end{remark}
\color{black}
Our next theorem provides regret guarantees for HI-LCB and HI-LCB-lite for i.i.d. stochastic arrivals. 
\begin{theorem}[Stochastic Arrivals]
    \label{thm:stoch}
     Under i.i.d. stochastic arrivals (Assumption \ref{as:stoch}) and i.i.d. offloading costs, for any constant $\alpha> 0.5$:
\begin{enumerate}
    \item[(a)] $\displaystyle R_{\text{HI-LCB}}^S(T) \leq \sum_{\phi_i \in \Phi_H} \; \min_{\phi_j \in\Phi_H^{(i)}} \left\{\frac{16\alpha w_i \Delta_{\phi_i} }{w_j\Delta^2_{\phi_j}}\;\right\}\log T+ \mathbf{C_3}(\Delta)$, 
    \item[(b)] $\displaystyle R_{\text{HI-LCB-lite}}^S(T) \leq \sum_{\phi_i \in \Phi_H} \; \frac{16\alpha}{\Delta_{\phi_i}}\log T\; +\mathbf{C_1}(\Delta)$.
\end{enumerate}
Further, for the special case with fixed and a priori known offloading cost discussed in Remark \ref{rem:fc_defn} and Remark \ref{rem:fc},
\begin{enumerate}
    \item[(c)] $\displaystyle R_{\text{HI-LCB}}^S(T) \leq \sum_{\phi_i \in \Phi_H} \; \min_{\phi_j \in\Phi_H^{(i)}} \left\{\frac{4\alpha w_i \Delta_{\phi_i} }{w_j\Delta^2_{\phi_j}}\;\right\}\log T+  \mathbf{C_4}(\Delta)$, 
    \item[(d)] $\displaystyle R_{\text{HI-LCB-lite}}^S(T) \leq \sum_{\phi_i \in \Phi_H} \; \frac{4\alpha}{\Delta_{\phi_i}}\log T\; +\mathbf{C_2}(\Delta)$.
\end{enumerate}
\end{theorem}

The stochastic arrival process enables HI-LCB to effectively utilize the non-decreasing nature of $f(\cdot)$, thereby reducing regret incurred due to the confidence measures in $\Phi_H$ that are seen rarely. Specifically, if some $\phi_j < \phi_i$ already has its LCB above the decision threshold, the policy accepts inferences with confidence measure $\phi_i \in \Phi_H$, even if $\phi_i$ has not yet accumulated sufficient observations on its own. Therefore, under Assumption \ref{as:stoch}, HI-LCB enjoys a tighter performance guarantee compared to that for adversarial arrivals in Theorem \ref{thm:adv}. Since HI-LCB-lite does not exploit the non-decreasing nature of $f(\cdot)$, its regret bounds for stochastic arrivals are identical to those for adversarial arrivals. For both policies, the guarantees for known fixed offloading costs are tighter than those for i.i.d. offloading costs with an unknown mean.

Our next result characterizes a lower bound on the regret of any online HI offloading policy. 
\begin{theorem}[Lower Bound on Regret] 
\label{thm:lb}
   Let $\mathcal{A}$ denote the class of all Hierarchical Inference policies that do not assume prior knowledge of the inference accuracy mapping $f(\cdot)$ or the mean offloading cost $\gamma$. Then, for any such policy $\pi \in \mathcal{A}$, the adversarial regret over time  $T$ satisfies:
\[
  \min_{\pi \in \mathcal{A}} \max_{\Phi,\gamma} R_\pi^A(T) = \Omega(\log T).  
\]
\end{theorem}

We thus conclude that HI-LCB and HI-LCB-lite have \textit{order-optimal} $O(\log T)$ regret performance with respect to the number of samples, i.e., the number of time-slots. 

 \color{black}
    \section{Proof Outlines}
    \label{sec:outlines}
Here, we present proof outlines for the results presented in Section \ref{sec:results}. The detailed proofs are in the appendix. 
Although there are similarities between our problem and contextual bandits (Remark \ref{rem:contextual}), our proofs leverage the specific structure of our problem and are entirely independent of the proofs for contextual bandits.

    \subsection{Proof of Theorem \ref{thm:adv}}
We first present the outline of the proof of Theorem \ref{thm:adv}(a) and Theorem \ref{thm:adv}(b). 

For a given confidence measure $\phi_i$, let $A_{\phi_i}(T)$ be the number of samples for which local inference is accepted up to time $T$, and $\mathsf{O}_{\phi_i}(T)$ be the number of samples offloaded by time $T$.
Recall that $\Delta_{\phi_i} = |1-f(\phi_i)-\gamma|$. For the samples with confidence measures in $\Phi_H$, the optimal action is to accept the inference; therefore, regret is incurred whenever these samples are offloaded. We prove the following:
\begin{lemma}
For any $\alpha>0.5$, the expected number of offloads for samples corresponding to a confidence measure $\phi_i \in \Phi_H$ till time $T$ for both HI-LCB and HI-LCB-lite is bounded by
\[\max_{\sigma_T}\mathds{E}[\mathsf{O}_{\phi_i}(T)] \leq \frac{16\alpha \log T}{\Delta_{\phi_i}^2} +\frac{4\alpha}{2\alpha-1}. \]
\end{lemma}
On the other hand, regret is incurred due to the samples with confidence measures in $\Phi_L$ only if the policy decides to accept the local inference. We show the following:
\begin{lemma}
    For any $\alpha>0.5$, the expected number of accepts by HI-LCB and HI-LCB-lite corresponding to all the confidence measures in $\Phi_L$ till time $T$ is bounded by
    \[ \max_{\sigma_T} \mathds{E}\left[ \sum_{\phi_i \in\Phi_L}A_{\phi_i}(T)\right]\leq (|\Phi_L|+1)\times \frac{2\alpha}{{2\alpha-1}}.\]
\end{lemma}
The next step is regret decomposition. 
Let $R^{\phi_i}_{\pi}(T)$ denote the regret for a policy $\pi$ associated with a given confidence measure $\phi_i \in \Phi$, i.e.,
\[
R^{\phi_i}_{\pi}(T)=\mathds{E}\left[ \sum_{t=1}^T\,(L_t^\pi- L_t^{\pi^\ast})\cdot\mathbb{1}\{\phi(t)=\phi_i\}\right].
\]
We show that 
\[  
R^{\phi_i}_{\pi}(T)=\begin{cases}(\gamma - 1 + f(\phi_i))\times \mathds{E}\left[\mathsf{O}_{\phi_i}(T)\right]& \text{if $\phi_i \in \Phi_H$},\\
(1-f(\phi_i) - \gamma)\times  \mathds{E}\left[A_{\phi_i}(T)\right]  & \text{otherwise}.
\end{cases}
\]    
We use this and the fact that
\[R^A_\pi(T)= \max_{\sigma_T} \sum_{\phi_i \in \Phi}R^{\phi_i}_{\pi}(T)
\]
\noindent to bound the regret for $\pi=$ HI-LCB, HI-LCB-lite, thus completing the proof.

The proofs of Theorem \ref{thm:adv}(c) and Theorem \ref{thm:adv}(d) use the same arguments as the proof of Theorem \ref{thm:adv}(a) and Theorem \ref{thm:adv}(b) under the additional simplification that there is no uncertainty in the value of the offload cost, leading to tighter regret bounds. 

\subsection{Proof of Theorem \ref{thm:stoch}}
We first present the proof outline for Theorem \ref{thm:stoch}(a). The stochastic arrival process allows HI-LCB to exploit the non-decreasing property of $f(\cdot)$, enabling better learning for rarely observed values of the confidence measures. Leveraging this, we prove that \[
   \mathds{E}[\mathsf{O}_{\phi_i}(T)]\leq w_i\cdot \min_{\phi_j \leq \phi_i} \mathds{E}\left[\frac{\mathsf{O}_{\phi_j}(T)}{w_j}\right],
    \]
where $w_i$ is the arrival probability of $\phi_i$. Consequently,  for HI-LCB with i.i.d. offloading costs, we show that
\[
   \mathds{E}[\mathsf{O}_{\phi_i}(T)] = O\left( \min_{\phi_j \in\Phi_H^{(i)}}\left(\frac{w_i}{w_j\Delta^2_{\phi_j}}\right)\log T\right)
    \text{for any $\phi_i$ in $\Phi_H$}. \]
The rest of the proof uses arguments similar to the proof of Theorem \ref{thm:adv}(a).


The proof for Theorem \ref{thm:stoch}(c) uses the same arguments as the proof of Theorem \ref{thm:stoch}(a) under the additional simplification that there is no uncertainty in the value of the offload cost, leading to tighter regret bounds. 

Since HI-LCB-lite does not exploit the non-decreasing property of $f(\cdot)$, we use the adversarial regret of the policy as an upper bound for the regret under stochastic arrivals, as stated in Theorems \ref{thm:stoch}(b) and \ref{thm:stoch}(d).

\subsection{Proof of Theorem \ref{thm:lb}}
\label{subsec:lb}
To prove the lower bound, we consider a simplified problem instance with a singleton confidence set $\Phi = \{ \phi_1 \}$, and the offloading cost $\Gamma_t$ drawn i.i.d. from a $\text{Bernoulli}(\gamma)$ distribution with $\gamma > 1 - f(\phi_1)$.
Let $\mathsf{O}_{\phi_1}(T)$ denote the number of offloads up to time $T$. Then, for any policy $\pi \in \mathcal{A}$ that guarantees a sublinear regret with high probability ($\geq1-1/T$),  we show that $\mathbb{E}[\mathsf{O}_{\phi_1}(T)]$ for this instance must satisfy:
 \begin{equation}
 \label{eq:lb-of}
  \lim_{T\rightarrow\infty} \mathbb{E}[\mathsf{O}_{\phi_1}(T)] \geq \frac{\log T} {D_B(\gamma  \,\|\, 1-f(\phi_1))},
 \end{equation}
where $D_B(p \,\|\, q)$ denotes the \textit{Kullback–Leibler divergence} \cite{lattimore2020bandit} between two Bernoulli distributions with parameters $p$ and $q$. The proof uses the change-of-measure argument used in the classical lower bound for multi-arm bandits \cite{LAI19854,lattimore2020bandit}. Using \eqref{eq:lb-of}, we have that $\min R_\pi^A(T)  = \Omega(\log T)$ for $\Phi=\{\phi_1\}$ with $\gamma>1-f(\phi_1)$, thus proving Theorem \ref{thm:lb}.

\vspace*{-0.1 in}
\section{Numerical Results}
\label{sec:simulations}
\vspace*{-0.1 in}
\color{black}
\begin{figure}[t]
    \centering
    \includegraphics[scale=0.34]{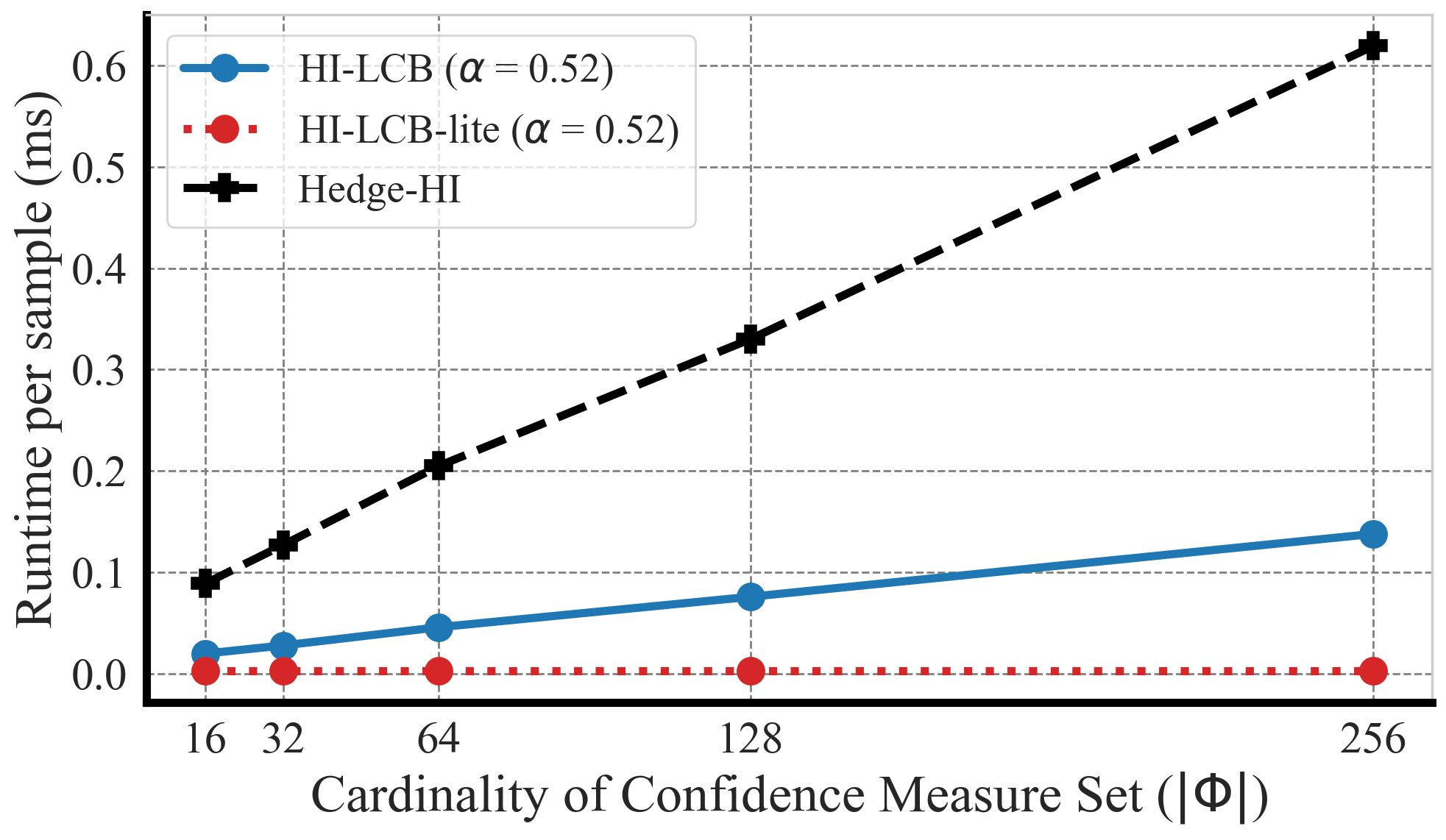}
    \caption{Runtime vs. cardinality of confidence measure set ($|\Phi|$).}
    \label{fig:runtimes}
    \vspace*{-0.2 in}
\end{figure}

\begin{figure*}[!t]
\centering
\begin{subfigure}[h]{1\textwidth}
    \includegraphics[width=\textwidth]{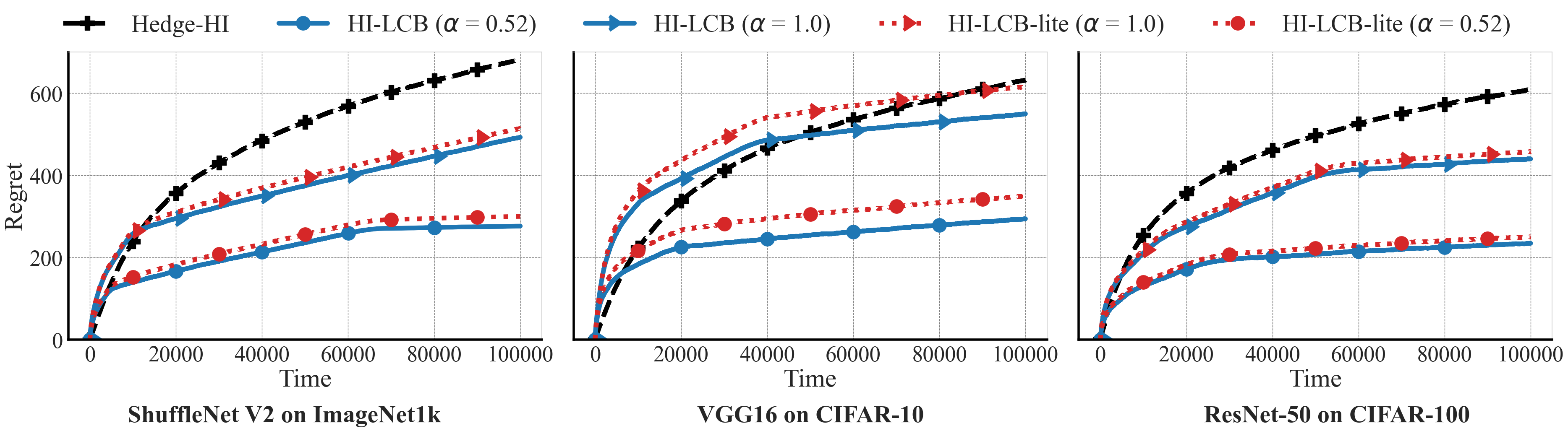}
    \caption{\normalsize Regret vs. time for a known constant offload cost $\gamma=0.5$.}
    \label{fig:accuracy1}
\end{subfigure}
\medskip 
\begin{subfigure}[h]{1\textwidth}
    \includegraphics[width=\textwidth]{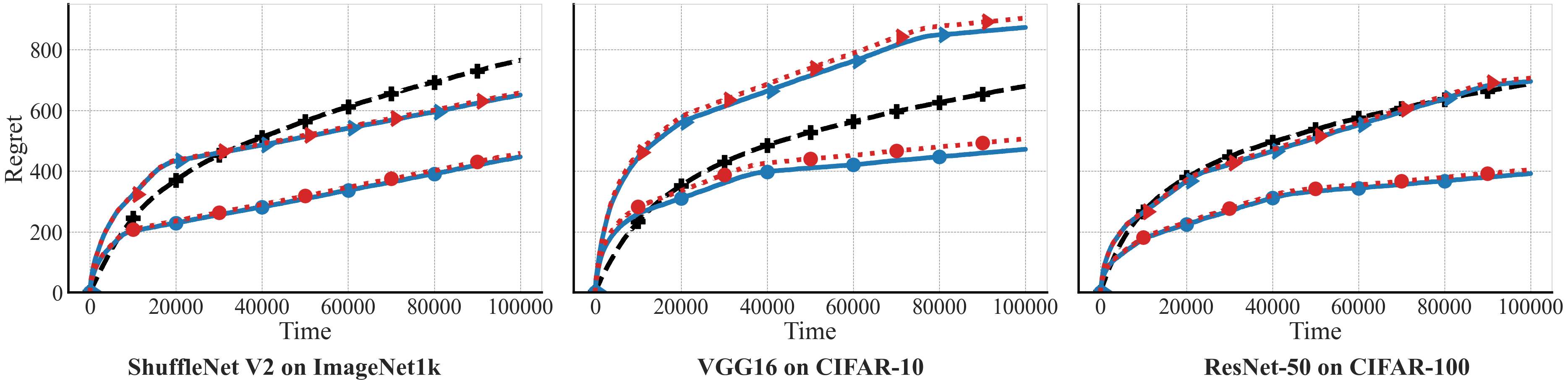}
    \caption{\normalsize Regret vs. time for i.i.d offload costs with $\mathds{E}[\Gamma_t]=0.5$.}
    \label{fig:accuracy2}
\end{subfigure}
\medskip 
\begin{subfigure}[h]{1\textwidth}
\includegraphics[width=\textwidth]{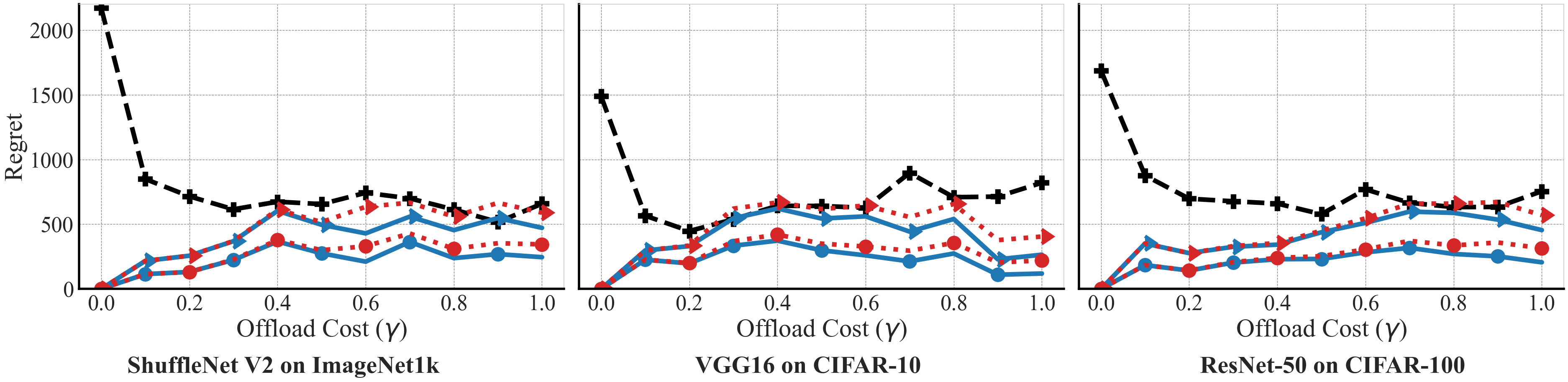}
    \caption{\normalsize Regret vs. offload cost ($\gamma$) for known constant offloading costs and $T=100000$.}
    \label{fig:accuracy3}
\end{subfigure}
\medskip 
\begin{subfigure}[h]{1\textwidth}
    \includegraphics[width=\textwidth]{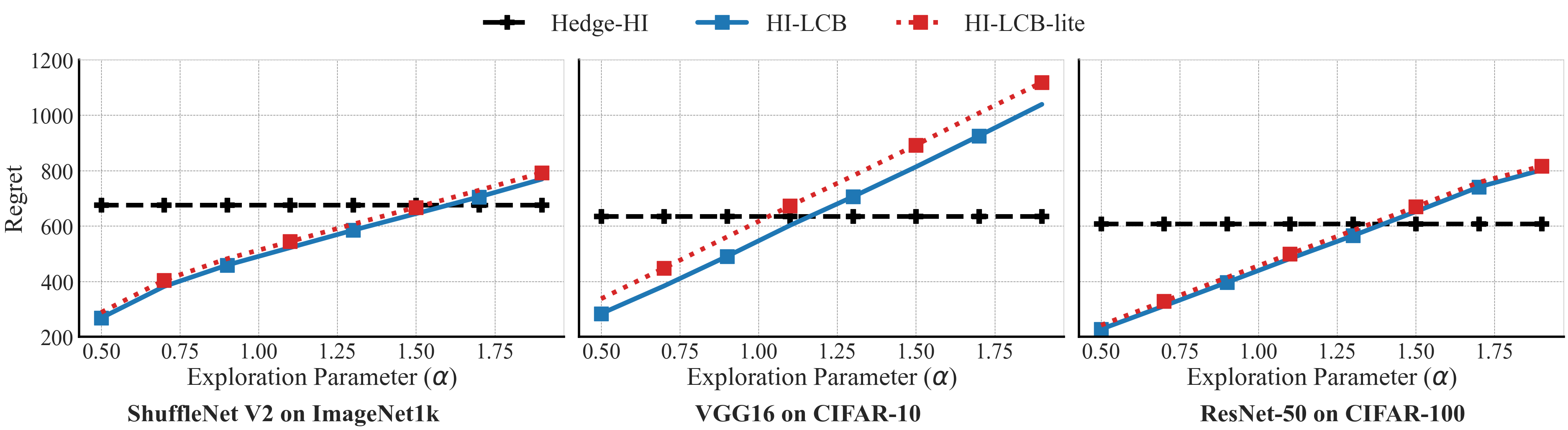}
    \caption{\normalsize Regret vs. exploration parameter $(\alpha)$ for known constant offloading cost $\gamma$ = 0.5 and $T=100000$.}
    \label{fig:accuracy4}
\end{subfigure}
\caption{Simulation results for ImageNet1k, CIFAR-10, and CIFAR-100. The legend for sub-figures (a), (b), and (c) is provided at the top of the page. The legend for sub-figure (d) is provided just above the sub-figure.}
\label{fig:accuracy}
\end{figure*}

In this section, we present our simulation results.

Previous works \cite{hilf,hedge-hi,behera2025exploring} have demonstrated that HI policies achieve superior trade-offs for latency and accuracy compared to two baseline policies: 1) the policy that accepts Local-ML inference for all samples, and 2) the policy that offloads all the samples for remote inference. Thus, in this work, we compare the performance of the HI policies proposed in this paper with that of existing HI policies.
\color{black}

Figure \ref{fig:runtimes} illustrates how the runtimes of the three policies vary with the cardinality of the set of all possible confidence measure values. Notably, the runtime of HI-LCB-lite remains constant, whereas the runtimes of the other two policies increase consistently with the cardinality of the confidence measure set.


We now present numerical results from experiments conducted on three standard image classification datasets: ImageNet1k \cite{deng2009imagenet}, CIFAR-10, and CIFAR-100 \cite{krizhevsky2009learning}. We employ several Local-ML models, including ShuffleNet V2, VGG16, and ResNet-50 \cite{he2016deep}. Our evaluation focuses on the performance of HI-LCB and HI-LCB-lite in comparison to the current state-of-the-art, Hedge-HI, with its hyperparameters set according to Corollary 2 of \cite{hedge-hi}. The performance metric used is regret relative to the optimal static threshold policy, $\pi^\ast$.

We present results for both known deterministic and stochastic offloading cost settings over a time horizon of $T = 100000$. All datasets are sampled uniformly at random to generate the arrival sequence. Each data point presented in the plot is averaged over 100 independent arrival sequences. We use the max-softmax function as the confidence measure, quantized to a 4-bit representation, resulting in $\lvert \Phi \rvert = 16$.

Recall that the regret guarantees in Section~\ref{sec:results} hold when the exploration parameter $\alpha$ is greater than $0.5$. We therefore restrict our numerical results to $\alpha > 0.5$. Figure \ref{fig:accuracy}(a) illustrates the regret performance of HI-LCB and HI-LCB-lite (with $\alpha=0.52$ and $\alpha=1$) alongside Hedge-HI for image classification with a known deterministic offloading cost of $\gamma=0.5$. \color{black}  Similarly, Figure \ref{fig:accuracy}(b) presents results for i.i.d. stochastic offloading costs, where $\Gamma_t$ is sampled from a bimodal distribution, taking values $0.45$ or $0.55$ with equal probability at each time-slot. 
At $T=100\,000$, HI-LCB and HI-LCB-lite with $\alpha = 0.52$ significantly outperform Hedge-HI across all three datasets. In both scenarios, Hedge-HI outperforms our policies for short time-horizons, and HI-LCB and HI-LCB-lite outperform Hedge-HI as the time increases. HI-LCB and HI-LCB-lite experience higher initial regret than Hedge-HI because, by definition, both policies offload samples across all confidence values until they can identify with high probability which values belong to $\phi_H$. After this initial phase, HI-LCB and HI-LCB-lite start outperforming Hedge-HI \color{black} and the gap between the regrets of Hedge-HI and our policies increases with time. As expected, HI-LCB marginally outperforms HI-LCB-lite for all parameter settings.

 In Figure \ref{fig:accuracy}(c), we compare the performance of various HI policies under constant and known offload costs spanning $[0,1]$. Notably, HI-LCB with $\alpha = 0.52$ consistently outperforms or matches the performance of Hedge-HI. HI-LCB-lite outperforms Hedge-HI for most parameter settings, with marginally worse performance in some cases. 

Figure \ref{fig:accuracy}(d) illustrates how the regrets for HI-LCB and HI-LCB-lite vary with the exploration parameter $\alpha$, with a fixed offloading cost of $\gamma = 0.5$. The regret of both policies consistently increases with $\alpha$. This suggests that smaller values of $\alpha$ (with $\alpha > 0.5$) result in lower regret.

Tables \ref{tab:offloads}  and \ref{tab:accuracy} show the fraction of samples offloaded and the fraction of samples classified accurately,  respectively, for the three policies with $\alpha = 0.52$, $\gamma = 0.5$, and $T = 100\,000$. We observe that across all three experimental setups, i.e., VGG16 on the 10-class CIFAR-10 dataset; ResNet-50 on the 100-class CIFAR-100 dataset; and ShuffleNet V2 on the 1000-class ImageNet1k dataset,
\color{black}HI-LCB-lite achieves the highest accuracy and fraction of offloaded samples, followed by HI-LCB and Hedge-HI.

\begin{table}[t]
    \centering
    \renewcommand{\arraystretch}{1.4}
    \begin{tabular}{cccc} 
        \hline 
        & ImageNet1k & CIFAR-10 & CIFAR-100\\ 
        \hline 
        Hedge-HI  & 0.361  & 0.325  & 0.493 \\
        HI-LCB  & 0.450  & 0.333 & 0.620\\
        HI-LCB-lite  & 0.451 & 0.337 & 0.624\\
        \hline 
    \end{tabular}
      \caption{Fraction of samples offloaded.}
      \label{tab:offloads}
\end{table}
\color{black}
\begin{table}[t]
    \centering
    \renewcommand{\arraystretch}{1.4}
    \begin{tabular}{cccc} 
        \hline 
        & ImageNet1k & CIFAR-10 & CIFAR-100\\ 
        \hline 
        Hedge-HI  & $86.38$  & $89.63$  & $85.42$\\
        HI-LCB  &  $91.14$ &  $90.29$& $92.28$\\
        HI-LCB-lite  & $91.17$ & $90.42$  & $92.43$\\
        \hline 
    \end{tabular}
      \caption{Classification accuracy (\%).}
      \label{tab:accuracy}
       \vskip -0.1 in
\end{table}
\color{black}

Overall, our policies offload more samples and achieve higher accuracy. Our performance metric rewards accuracy while penalizing excessive offloads, and thus, the total cost incurred by our policies is lower than Hedge-HI, resulting in lower regret. Furthermore, a key takeaway from our results is that HI-LCB-lite achieves comparable performance to HI-LCB, despite its significantly lower computational complexity.

\color{black}
\section{Conclusions}

This work focuses on the HIL problem. A central modeling contribution of this work is the structural relationship between a Local-ML model's confidence measure for a given sample and the probability of its inference being correct. This relationship is empirically supported by data and enabled the design of two novel online policies that use the Upper Confidence Bound (UCB) framework: HI-LCB and HI-LCB-lite.  While these policies leverage the Upper Confidence Bound (UCB) framework which is a staple of multi-armed bandit (MAB) literature, our setting diverges due to a specific feedback asymmetry: the offload arm reveals rewards for both arms, whereas the accept arm reveals nothing. This unique structure necessitates the novel policies and analysis. \color{black} Our proposed policies offer distinct advantages over existing approaches for the HIL problem, including tighter theoretical regret guarantees, significantly reduced computational complexity for HI-LCB-lite, and superior empirical performance. Additional noteworthy aspects are that HI-LCB and HI-LCB-lite are anytime policies, i.e., they do not need to know the time horizon ($T$) upfront. Existing policies, however, need to know $T$ to optimize their hyperparameters. Further, HI-LCB and HI-LCB-lite are deterministic policies, thereby eliminating the requirement to generate random quantities during execution.

\bibliographystyle{IEEEtran} 
\bibliography{IEEEabrv, ref}

\section*{Appendix A}
\subsection{Proof of Theorem \ref{thm:adv}(a) and Theorem \ref{thm:adv}(b)}
Corresponding to each confidence measure $\phi_i \in \Phi$, we define the following quantities:
\begin{itemize}[itemsep=0.5em]
    \item[--] $\mathsf{O}_{\phi_i}(t)$: Value of $\mathsf{O}_{\phi_i}$ at the start of time-slot $t$,
    \item[--] $\hat{f}(\phi_i, t)$: Value of $\hat{f}(\phi_i)$ at the start of time-slot $t$,
    \item[--] $\text{LCB}_{\phi_i}(t)$: Value of $\text{LCB}_{\phi_i}$ at the start of time-slot $t$,
    \item[--] $A_{\phi_i}(t)$: Number of accepted samples until time $t$ with confidence measure $\phi_i$, i.e.,
\[A_{\phi_i}(t)=\sum_{n=1}^t\mathbb{1}\left\{D_\pi(n)=0, \phi(n)=\phi_i\right\}.\]
\end{itemize}
Additionally, for the average offloading cost $\gamma$, we define:
\begin{itemize}[itemsep=0.5em]
    \item[--] $\hat{\gamma}(t)$: Value of estimator $\hat{\gamma}$ at the start of time-slot $t$,
    \item[--] $\text{LCB}_\gamma( t)$: Value of $\text{LCB}_\gamma$ at the start of time-slot $t$.
    \item[--] $\mathsf{O}_{\gamma}(t)$: Value of $\mathsf{O}_{\gamma}$ at the start of time-slot $t$ Here, \(\displaystyle \mathsf{O}_{\gamma}(t) =\sum_{\phi_i \in \Phi} \mathsf{O}_{\phi_i}(t)\).
\end{itemize}
Lastly, we define the following events:
\begin{equation}
\begin{aligned}
 \label{eq:Ait}   
G_{i,t}&:= 1-\text{LCB}_{\phi_i}(t)\geq \text{LCB}_\gamma(t),\\
H_{i,t}&:=  \mathsf{O}_{\phi_i}(t) > 16\alpha\log t/\Delta^2_{\phi_i},\\
I_t&:= {\bigcap_{\Phi_L}\overline{G_{j,t}}}: \max_{\Phi_L}\left\{1-\text{LCB}_{\phi_j}(t)\right\} < \text{LCB}_\gamma(t).
\end{aligned}
\end{equation}

 \begin{lemma}[Hoeffding's inequality \cite{lattimore2020bandit}]
 \label{lem:hoeff}
Let $M_1$, $M_2$, $\dots$, $M_n$ $\in  [0, b]$ be i.i.d. random variables with \( \mu = \mathbb{E}[M_i] \) for all \( i \in \{1,2,\dots n\}\). Let 
\(\displaystyle
\hat{\mu}(n) = \frac{1}{n} \sum_{i=1}^n M_i.
\) 
For any \( \epsilon > 0 \),
\[
\Pr(|\hat{\mu}(n) - \mu| \geq \epsilon) \leq 2 e^{-2n\epsilon^2/b^2},
\]
\end{lemma}

We use Lemma \ref{lem:Ait} and Lemma \ref{lem:Offload-H} to characterize an upper bound on the adversarial regret of HI-LCB and HI-LCB-lite associated with the confidence measures in $\Phi_H$.
 \begin{lemma}
\label{lem:Ait}
 Under HI-LCB and HI-LCB-lite with i.i.d offloading costs, for any $\phi_i \in \Phi_H$,
\begin{equation}
    \label{eq:pr-ait}
\Pr(G_{i,t} | {H_{i,t}}) \leq 2t^{-2\alpha} \text{ for any } t \in\{1,\dots, T\}.
\end{equation}
\end{lemma}
\begin{proof}
From the $\text{LCB}_\phi$ definitions in \eqref{eq:LCB_def} and \eqref{eq:vanillaLCB}, and the $\text{LCB}_\gamma$ definition in \eqref{eq:LCB2} we know that for any $\phi_i \in \Phi$, under both HI-LCB and HI-LCB-lite  at any time $t$:
\begin{equation}
\label{eq:lcb-comp}
    \begin{aligned}
    \text{LCB}_{\phi_i}(t) &\geq\hat f(\phi_i,t) - \sqrt{\frac{\alpha \log t}{\mathsf{O}_{\phi_i}(t)}},\\
    \text{LCB}_{\gamma}(t)&=\hat{\gamma}(t)-\sqrt{\frac{\alpha \log t}{\mathsf{O}_{\gamma}(t)}}\geq\hat{\gamma}(t)-\sqrt{\frac{\alpha \log t}{\mathsf{O}_{\phi_i}(t)}}.
    \end{aligned}
\end{equation}

Applying Hoeffding's inequality (Lemma \ref{lem:hoeff}), we obtain
\begin{equation}
    \label{eq:interim-hoeff}
   \Pr \left(f(\phi_i)-\hat{f}(\phi_i,t)>\epsilon\right)\leq e^{-2\epsilon^2\mathsf{O}_{\phi_i}(t)} .
\end{equation}
Setting \(\displaystyle \epsilon= \sqrt{\frac{\alpha \log t}{\mathsf{O}_{\phi_i}(t)}}\) in \eqref{eq:interim-hoeff}, we get
\begin{equation}
\begin{aligned}
\label{eq:hoeff1}
&\Pr \left(\hat{f}(\phi_i,t)<f(\phi_i)-\sqrt{\frac{\alpha \log t}{\mathsf{O}_{\phi_i}(t)}}\right)\leq t^{-2\alpha}, \text{ and thus,}\\
&\Pr\left(\hat{f}(\phi_i,t)-\sqrt{\frac{\alpha \log t}{\mathsf{O}_{\phi_i}(t)}}<f(\phi_i) - \sqrt{\frac{4\alpha \log t}{\mathsf{O}_{\phi_i}(t)}}\right)\leq t^{-2\alpha},\\
&\Pr\left(\text{LCB}_{\phi_i}(t)<f(\phi_i) - \sqrt{\frac{4\alpha \log t}{\mathsf{O}_{\phi_i}(t)}}\right)\leq t^{-2\alpha}, \text{and}\\
&\Pr\left(1-\text{LCB}_{\phi_i}(t)>1-f(\phi_i) + \sqrt{\frac{4\alpha \log t}{\mathsf{O}_{\phi_i}(t)}}\right)\leq t^{-2\alpha}.
\end{aligned}
\end{equation}
Similarly, for the offloading costs:
\begin{equation}
\begin{aligned}
\label{eq:hoeff2}
    &\Pr\left(\hat{\gamma}(t)-\sqrt{\frac{\alpha \log t}{\mathsf{O}_{\gamma}(t)}}<\gamma - \sqrt{\frac{4\alpha \log t}{\mathsf{O}_{\gamma}(t)}}\right)\leq t^{-2\alpha},\\
    &\Pr\left(\text{LCB}_{\gamma}(t)<\gamma - \sqrt{\frac{4\alpha \log t}{\mathsf{O}_{\gamma}(t)}}\right)\leq t^{-2\alpha}.
\end{aligned}\end{equation}

 For any $\phi_i \in \Phi_H$, given $\mathsf{O}_{\phi_i}(t) > 16\alpha \log t/\Delta_{\phi_i}^2$,
{\small
\begin{equation}
\label{eq:uneq}
1-f(\phi_i) + \sqrt{\frac{4\alpha \log t}{\mathsf{O}_{\phi_i}(t)}} < \gamma-\sqrt{\frac{4\alpha \log t}{\mathsf{O}_{\phi_i}(t)}} \leq \gamma-\sqrt{\frac{4\alpha \log t}{\mathsf{O}_{\gamma}(t)}}.
\end{equation}}
Using \eqref{eq:hoeff1}, \eqref{eq:hoeff2}, and \eqref{eq:uneq}, we obtain
\[
\Pr\left(1-\text{LCB}_{\phi_i}(t)>\text{LCB}_{\gamma}(t)\,|\,{H_{i,t}}\right) \leq t^{-2\alpha}+t^{-2\alpha}.\\
\]
This completes the proof of Lemma \ref{lem:Ait}.
\end{proof}

\begin{lemma}
    \label{lem:Offload-H}
For any $\alpha>0.5$, the expected number of offloads for samples corresponding to confidence measure $\phi_i \in \Phi_H$ for both HI-LCB and HI-LCB-lite is bounded by
\[\max_{\sigma_T}\mathds{E}[\mathsf{O}_{\phi_i}(t)] \leq \frac{16\alpha \log t}{\Delta_{\phi_i}^2} +\frac{4\alpha}{2\alpha-1}. \]
\end{lemma}
\begin{proof} 
{Using the \textit{Law of Total Expectation} \cite{lattimore2020bandit}, we obtain
\[
\begin{aligned}
\mathds{E}\left[ \mathbb{1}\{G_{i,n}\}\right]= &\mathds{E}\left[ \mathbb{1}\{G_{i,n}\}\mid H_{i,n}\right]\cdot\Pr(H_{i,n})\\&+ \mathds{E}\left[ \mathbb{1}\{G_{i,n}\}\mid \overline{H_{i,n}}\right]\cdot\Pr(\overline{H_{i,n}})\\
\leq& \mathds{E}\left[ \mathbb{1}\{G_{i,n}\}\mid H_{i,n}\right]+ \mathds{E}\left[ \mathbb{1}\{G_{i,n}\}\mid \overline{H_{i,n}}\right].
\end{aligned}\]}

 We show that for any confidence measure $\phi_i \in \Phi_H$, under any arrival sequence $\sigma_T$, $ \mathds{E}[\mathsf{O}_{\phi_i}(t)]$ can be bounded in the following manner using Lemma \ref{lem:Ait}.
\[
\begin{aligned}
    \mathds{E}[\mathsf{O}_{\phi_i}(t)]&= \mathds{E}\left[\sum_{n=1}^t \mathbb{1}\{G_{i,n},\,\phi(n) =\phi_i \}\right]\\&=\sum_{n=1}^t\mathds{E}\left[\mathbb{1}\{G_{i,n},\,\phi(n) =\phi_i\}\right]
    \\&\leq\sum_{n=1}^t\mathds{E}\left[\mathbb{1}\{G_{i,n}\}\right]\\
    &\leq \sum_{n=1}^t\mathds{E}\left[\mathbb{1}\{G_{i,n}\}|\overline{H_{i,n}}\right] + \sum_{n=1}^t\mathds{E}\left[ \mathbb{1}\{G_{i,n}\}|{H_{i,n}}\right]\\
    &\leq \sum_{n=1}^t\mathds{E}\left[ \mathbb{1}\{G_{i,n}\}|\overline{H_{i,n}}\right] + \sum_{n=1}^t\Pr\left( G_{i,n}|{H_{i,n}}\right)\\
    &\leq \sum_{n=1}^t\mathds{E}\left[\mathbb{1}\{G_{i,n}\}|\overline{H_{i,n}}\right] + 2\sum_{n=1}^t n^{-2\alpha}\\
    &\leq \frac{16\alpha \log t}{\Delta_{\phi_i}^2} + 2\zeta(2 \alpha), 
\end{aligned}
\]
where $\zeta(\cdot)$ is the \textit{Riemann zeta} function. For any $k>1$, $\zeta(k)$ is upper-bounded by $\frac{k}{k-1}$, thus proving Lemma \ref{lem:Offload-H}.
\end{proof}

We use Lemma \ref{lem:Bit} and Lemma \ref{lem:Accept-L} to characterize an upper bound for the adversarial regret of HI-LCB and HI-LCB-lite associated with the confidence measures in $\Phi_L$.
\begin{lemma}
\label{lem:Bit} For HI-LCB and HI-LCB-lite with i.i.d. offloading costs,
at any time $t\leq T$, \(
\Pr(I_t) \leq  (|\Phi_L|+1)t^{-2\alpha}.\)
\end{lemma}
\begin{proof}
To prove Lemma \ref{lem:Bit}, we define the following events:
\begin{equation}
\label{eq:bt12}
    \begin{aligned}
I^{(1)}_{t}:=& \,\hat{f}(\phi_j,t) \leq f(\phi_j) + \sqrt{\frac{\alpha \log t}{\mathsf{O}_{\phi_j}(t)}}\;\forall\; \phi_j \in\Phi_L,\\
I^{(2)}_t:=& \, \hat{\gamma}(t) \leq \gamma - \sqrt{\frac{\alpha \log t}{\mathsf{O}_{\gamma}(t)}}.
    \end{aligned}
\end{equation}
From Hoeffding's inequality (Lemma \ref{lem:hoeff}) and the union bound, we obtain 
\[\Pr\left(I^{(1)}_{t}\right)\geq1-{|\Phi_L|}\,{t^{-2\alpha}} \;\text{and }\Pr\left(I^{(2)}_{t}\right)\geq1-t^{-2\alpha}.\]
Given event $I^{(1)}_{t}$ occurs, $1-\text{LCB}_{\phi_j}(t)\geq 1-f(\phi_j)$ for all $\phi_j\in \Phi_L$, and  given $I^{(2)}_{t}$ occurs, $\text{LCB}_\gamma(t) \leq\gamma < 1-f(\phi_j)$ for all $ \phi_j\in \Phi_L$. Therefore, using the definition of $I_t$:
\[
\Pr(I_t )\leq \Pr\left(\overline{I^{(1)}_{t}\cap I^{(2)}_{t}}\right)\leq \Pr\left(\overline{I^{(1)}_{t}}\right)+\Pr\left(\overline{I^{(2)}_{t}}\right)\leq (|\Phi_L|+1){t^{-2\alpha}}.
\]
This proves the Lemma \ref{lem:Bit}.
\end{proof}

\begin{lemma}
\label{lem:Accept-L}
    The expected number of inferences accepted by HI-LCB and HI-LCB-lite corresponding to all the confidence measures in $\Phi_L$ till time t is bounded by
    \[ \max_{\sigma_T} \mathds{E}\left[ \sum_{\phi_i \in\Phi_L}A_{\phi_i}(t)\right]\leq (|\Phi_L|+1)\times \frac{2\alpha}{{2\alpha-1}}.\]
\end{lemma}
\begin{proof}
A sample corresponding to a confidence measure in $\Phi_L$ is accepted by HI-LCB and HI-LCB-lite at slot $n$ only if the event {$K_n$} had occurred; therefore using Lemma \ref{lem:Bit}, we show that under any arrival sequence $\sigma_T$:
\begin{equation}
\label{eq: reg-low}
\begin{aligned}
    \mathds{E}\left[ \sum_{\phi_i \in\Phi_L}A_{\phi_i}(t)\right]
    &\leq \mathds{E}\left[ \sum_{n=1}^t\mathbb{1}{\{I_n \}}\right] \\ &=  \sum_{t=1}^t\Pr\left(I_n\right) \\
    &\leq (|\Phi_L|+1)  \cdot\sum_{n=1}^t  n^{-2\alpha}\
   \\ &\leq (|\Phi_L|+1)  \cdot  \left(\frac{2\alpha}{{2\alpha-1}}\right).
    \end{aligned}
\end{equation}
This completes the proof of Lemma \ref{lem:Accept-L}.
\end{proof}

\begin{proof}[Proof of Theorem \ref{thm:adv}(a) and Theorem \ref{thm:adv}(b)] 
Given a policy $\pi$, the expected difference in cost incurred between the given policy and the optimal policy $\pi^\ast$, conditioned on their offloading decisions differing for a sample arriving at time $t$ with confidence measure $\phi(t)=\phi_i$, is equal to the quantity $\Delta_{\phi_i}$ defined in Section \ref{sec:policies}, i.e.,
\begin{equation}
\begin{aligned}
\Delta_{\phi_i}=& \mathds{E}\left[L_t^\pi-L_t^{\pi^\ast}\,|\,D_{\pi}(t)\neq D_{\pi^\ast}(t), \phi(t)=\phi_i\right] \\
=&|1-f(\phi_i)-\gamma|.
\end{aligned}
\end{equation}

Based on the regret definitions in Section \ref{sec:perf}, we define $R^{\phi_i}_{\pi}(T)$ as the regret for a policy $\pi$ corresponding to a given confidence measure $\phi_i$, i.e., 
\begin{equation}
\begin{aligned}
R^{\phi_i}_{\pi}(T)=& \mathds{E}\left[ \sum_{t=1}^T\,(L_t^\pi- L_t^{\pi^\ast})\cdot\mathbb{1}\{\phi(t)=\phi_i\}\right]\\
=&\Delta_{\phi_i}\cdot \mathds{E}\left[ \sum_{t=1}^T\,\mathbb{1}\left\{\phi(t)=\phi_i,D_{\pi}(t)\neq D_{\pi^\ast}(t)\right\}\right].
\end{aligned}
\end{equation} 
From the characterization of $\pi^\ast$ in Lemma \ref{lem:opt}, we obtain
\begin{equation}    
R^{\phi_i}_{\pi}(T)=\begin{cases}\Delta_{\phi_i}\times \mathds{E}\left[\mathsf{O}_{\phi_i}(T)\right]& \text{if $\phi(t) \in \Phi_H$},\\
\Delta_{\phi_i}\times  \mathds{E}\left[A_{\phi_i}(T)\right]  & \text{otherwise}.
\end{cases}
\end{equation}  
We analyze the total adversarial regret by summing the regrets incurred due to the individual confidence measures in $\Phi_H$ and $\Phi_L$ in the following manner:
{ \[
    \begin{aligned}
        R_{\pi}^A(T)=& \max_{\sigma_T} \sum_{\phi_i \in \Phi}R^{\phi_i}_{\pi}(T)\\
        =& \max_{\sigma_T} \left\{\sum_{\phi_i \in \Phi_H}R^{\phi_i}_{\pi}(T) + \sum_{\phi_i \in \Phi_L}R^{\phi_i}_{\pi}(T)\right\} \\
        =& \max_{\sigma_T} \left\{\sum_{\phi_i \in \Phi_H}\Delta_{\phi_i}\, \mathds{E}\left[\mathsf{O}_{\phi_i}(T)\right] + \sum_{\phi_i \in \Phi_L}\Delta_{\phi_i}\,  \mathds{E}\left[A_{\phi_i}(T)\right]\right\} \\
        \leq& \max_{\sigma_T} \left\{\sum_{\phi_i \in \Phi_H}\Delta_{\phi_i}\, \mathds{E}\left[\mathsf{O}_{\phi_i}(T)\right]\right\}  \\   & + \max_{\sigma_T}\left\{\max_{\Phi_L}\Delta_{\phi_i}\cdot \mathds{E}\left[\sum_{\phi_i \in \Phi_L}A_{\phi_i}(T)\right]\right\}.
    \end{aligned}
\]}
 
Using Lemma \ref{lem:Offload-H} and Lemma \ref{lem:Accept-L}, we obtain the following upper bound on the cumulative regret for both the policies $\pi=\text{HI-LCB}$ and $\pi=\text{HI-LCB-lite}$ for $T$ samples under any arrival sequence and i.i.d offloading costs:
\[
\begin{aligned}
&R_{\pi}^A(T) \leq \sum_{\phi_i \in \Phi_H} \; \frac{16\alpha}{\Delta_{\phi_i}}\log T\; +\mathbf{C}(\Delta), \text{with}\\
    &\mathbf{C}(\Delta) = \sum_{\phi_i \in \Phi_H} \;\left( \frac{4\alpha \Delta_{\phi_i}}{2\alpha-1}  \right)+ \max_{ \Phi_L}\; 2\alpha\Delta_{\phi_j} \left(\frac{|\Phi_L|+1 }{{2\alpha-1}}\right).
    \end{aligned}
\]
This completes the proof of Theorem \ref{thm:adv}(a) and Theorem \ref{thm:adv}(b).
\end{proof}

\subsection{Proof of Theorem \ref{thm:adv}(c) and Theorem \ref{thm:adv}(d)} 
For the setup with a constant offloading cost $\gamma$, discussed in Remark \ref{rem:fc}, the analysis proceeds in a manner similar to that for i.i.d. costs. We define the following events:
\begin{equation}
\begin{aligned}
 \label{eq:G'it}   
J_{i,t}&:=1-\text{LCB}_{\phi_i}(t)\geq \gamma,\\
K_{i,t}&:= \mathsf{O}_{\phi_i}(t) > 4\alpha\log t/\Delta^2_{\phi_i},\\
N_t&:=\bigcap_{\Phi_L}\overline{J_{j,t}}: \max_{\Phi_L}\left\{1-\text{LCB}_{\phi_j}(t)\right\} < \gamma.
\end{aligned}
\end{equation}

We use Lemma \ref{lem:Ait2} and Lemma \ref{lem:Offload-H2} to characterize an upper bound on the adversarial regret of HI-LCB and HI-LCB-lite associated with the confidence measures in $\Phi_H$ in a setup with constant offloading cost.
\begin{lemma}
\label{lem:Ait2}
 Under HI-LCB and HI-LCB-lite with constant offloading cost, for any $\phi_i \in \Phi_H$,
\[
\Pr(J_{i,t} | {K_{i,t}}) \leq t^{-2\alpha} \text{ for any } t \in\{1,\dots, T\}.
\]
\end{lemma}
{\color{black} The proof of Lemma \ref{lem:Ait2} uses arguments similar to the proof of Lemma \ref{lem:Ait}.}
\begin{lemma}
    \label{lem:Offload-H2}
For any $\alpha>0.5$, the expected number of offloads for samples corresponding to any confidence measure $\phi_i \in \Phi_H$ for HI-LCB and HI-LCB-lite with known fixed offloading cost is bounded by
\[\max_{\sigma_T}\mathds{E}[\mathsf{O}_{\phi_i}(t)] \leq \frac{4\alpha \log t}{\Delta_{\phi_i}^2} +\frac{2\alpha}{2\alpha-1}. \]
\end{lemma}
\begin{proof}
Using Lemma \ref{lem:Ait2}, we upper bound the expected number of offloads any confidence measure $\phi_i \in \Phi_H$ under any arrival sequence $\sigma_T$ in a manner similar to what was done in Lemma \ref{lem:Offload-H} with
\[
\begin{aligned}
    \mathds{E}[\mathsf{O}_{\phi_i}(t)]&= \mathds{E}\left[\sum_{n=1}^t \mathbb{1}\left\{J_{i,n},\,\phi(n) =\phi_i\right\}\right]\\
    &\leq \mathds{E}\left[\sum_{n=1}^t \mathbb{1}\{J_{i,n}\}\right]\\
    &\leq \sum_{n=1}^t\mathds{E}\left[\mathbb{1}\{J_{i,n}\}|\overline{K_{i,n}}\right] + \sum_{n=1}^t\Pr(J_{i,n}|{K_{i,n}})\\
    &\leq \frac{4\alpha \log t}{\Delta_{\phi_i}^2} + \frac{2\alpha}{2\alpha-1}. 
\end{aligned}
\]

This proves the Lemma \ref{lem:Offload-H2}.
\end{proof}
We use Lemma \ref{lem:Bit2} and Lemma \ref{lem:Accept-L2} to characterize an upper bound for the adversarial regret of HI-LCB and HI-LCB-lite associated with the confidence measures in $\Phi_L$ under a constant offloading cost.
\begin{lemma}
\label{lem:Bit2} For HI-LCB and HI-LCB-lite with known constant offloading cost,
at any time $t\leq T$, \(
\Pr(N_t) \leq  |\Phi_L|t^{-2\alpha}.\)
\end{lemma}
{\color{black} The proof of Lemma \ref{lem:Bit2} uses arguments similar to the proof of Lemma \ref{lem:Bit}.}
\begin{lemma}
    \label{lem:Accept-L2}
    The expected aggregate number of accepts corresponding to the samples with confidence measures in $\Phi_L$ with constant offloading costs till time t is bounded by
    \[ \max_{\sigma_T} \mathds{E}\left[ \sum_{\phi_i \in\Phi_L}A_{\phi_i}(t)\right]\leq \frac{2\alpha |\Phi_L| }{{2\alpha-1}}.\]
\end{lemma}

\begin{proof}  
Using Lemma \ref{lem:Bit2},  we bound the expected aggregate number of accepts corresponding to the samples with confidence measures in $\Phi_L$ for any arrival sequence $\sigma_T$ with 
\[
   \mathds{E}\left[ \sum_{\phi_i \in \Phi_L}A_{\phi_i}(T)\right]
    \leq  \sum_{t=1}^T\Pr\left(N_t\right)\leq    \left(\frac{2\alpha|\Phi_L|}{{2\alpha-1}}\right).
\]
This completes the proof for Lemma \ref{lem:Accept-L2}.
\end{proof}
\begin{proof}[Proof of Theorem \ref{thm:adv}(c) and Theorem \ref{thm:adv}(d)]
For any policy $\pi$ operating for a time  $T$, we know that
 \[ 
 \begin{aligned}
R_{\pi}^A(T)\leq \max_{\sigma_T}  & \left\{\sum_{\phi_i \in \Phi_H}\Delta_{\phi_i}\, \mathds{E}\left[\mathsf{O}_{\phi_i}(T)\right]\right\} +\\   &\max_{\sigma_T}\left\{\max_{\Phi_L}\Delta_{\phi_i}\cdot\mathds{E}\left[\sum_{\phi_i \in \Phi_L}A_{\phi_i}(T)\right]\right\}.
\end{aligned}
\]

Using Lemma \ref{lem:Offload-H} and Lemma \ref{lem:Accept-L}, we obtain the upper bound on the cumulative regret for both  $\pi=\text{HI-LCB}$ and $\pi=\text{HI-LCB-lite}$ for $T$ samples under any adversarial arrival sequence and a known fixed offloading cost as
\[
\begin{aligned}
&R_{\pi}^A(T) \leq \sum_{\phi_i \in \Phi_H} \; \frac{4\alpha}{\Delta_{\phi_i}}\log T\; +\mathbf{C_2}(\Delta), \text{ where}\\
    &\mathbf{C_2}(\Delta) = \frac{2\alpha }{2\alpha-1} \left(\sum_{\phi_i \in\Phi_H} \;\Delta_{\phi_i} + |\Phi_L|\max_{ \Phi_L}\,\Delta_{\phi_j} \right).
    \end{aligned}
\]
Hence proving Theorem \ref{thm:adv} (c) and Theorem \ref{thm:adv} (d).
\end{proof}
\subsection{Proof of Theorem \ref{thm:stoch}}
{
Our focus herein will be on the policy $\pi = \text{HI-LCB}$ as the regret bounds for HI-LCB-lite under a stochastic arrival process (Theorem \ref{thm:stoch} (b), (d)) are identical to those under the adversarial settings (Theorem \ref{thm:adv} (b), (d)).

We define the event $\Lambda_{i,t}$ for a given time-slot $t$ and confidence measure $\phi_i \in \Phi$  in the following manner:
\begin{equation}
\label{eq:lambda-def}
    \Lambda_{i,t}:= 1-\text{LCB}_{\phi_i}(t)\geq \text{LCB}_\gamma(t).
\end{equation}
We use Lemma \ref{lem:Offload-H3} and Lemma \ref{lem:iid} to provide a regret upper-bound for HI-LCB with stochastic arrivals and i.i.d. offloading costs.
\begin{lemma}
    \label{lem:iid}
    Under Assumption \ref{as:stoch}, for any $\phi_i \in \Phi$,
    \[\mathds{E}[\mathsf{O}_{\phi_i}(t)]= w_i\cdot\mathds{E}\left[ \sum_{n=1}^t \mathbb{1}\{\Lambda_{i,n}\}\right]\]
\end{lemma}
\begin{proof}
Using the definition of $\mathsf{O}_{\phi_i}(t)$, under Assumption \ref{as:stoch} for the i.i.d stochastic arrival process, we have
\[
\begin{aligned}
\mathds{E}[\mathsf{O}_{\phi_i}(t)]&= \mathds{E}\left[ \sum_{n=1}^t \mathbb{1}\{\Lambda_{i,n},\phi(n)=\phi_i\}\right]\\
= &\mathds{E}\left[ \sum_{n=1}^t \mathbb{1}\{\Lambda_{i,n}\}\cdot\mathbb{1}\{\phi(n)=\phi_i\}\right]\\
=&\mathds{E}\left[ \sum_{n=1}^t \mathbb{1}\{\Lambda_{i,n}\}\right]\cdot \mathds{E}\left[\mathbb{1}\{\phi(n)=\phi_i\}\right]\\
=& w_i\cdot\mathds{E}\left[ \sum_{n=1}^t \mathbb{1}\{\Lambda_{i,n}\}\right]
\end{aligned}
\]
This completes the proof for Lemma \ref{lem:iid}.\end{proof}

\begin{lemma}
\label{lem:Offload-H3}
For any $\alpha>0.5$, the expected number of offloads for samples corresponding to any confidence measure $\phi_i \in \Phi_H$ for HI-LCB with the stochastic arrival sequence provided in Assumption \ref{as:stoch} can be upper-bounded by
\[ \mathds{E}[\mathsf{O}_{\phi_i}(T)] \leq w_i \cdot \min_{\Phi_H^{(i)}}\left( \frac{16\alpha\log T}{\Delta_{\phi_j}^2 w_j} + \frac{4\alpha w_j^{-1}}{2\alpha-1} \right).
\]
\end{lemma}
\begin{proof}
From the definition of $\Lambda_{i,t}$ in \eqref{eq:lambda-def}, and using the non-decreasing nature of $\text{LCB}_{\phi_i}(t)$ in HI-LCB, specifically with $\displaystyle\text{LCB}_{\phi_i}(t) \geq \max_{j \leq i} \text{LCB}_{\phi_j}(t)$, we obtain \[\mathbb{1}\{\Lambda_{i,t}\} \leq \min_{j \leq i} \mathbb{1}\{\Lambda_{j,t}\} \text{ for any $\phi_i \in \Phi$ and $t \in \{1,2,\dots T\}$}.\]

Using Lemma \ref{lem:Offload-H} and Lemma \ref{lem:iid}, we upper-bound the expected number of arrivals for a confidence measure $\phi \in \Phi_H$ for HI-LCB in the following manner:
\[
\begin{aligned}
\mathds{E}[\mathsf{O}_{\phi_i}(T)] &= w_i\cdot\mathds{E}\left[ \sum_{n=1}^t \mathbb{1}\{\Lambda_{i,n}\}\right]\\
&\leq  w_i\cdot \min_{j\leq i,\phi_j \in\Phi_H} \mathds{E}\left[ \sum_{n=1}^t \mathbb{1}\{\Lambda_{j,n}\}\right]\\
&\leq  w_i\cdot \min_{\Phi_H^{(i)}}\left( \max_{\sigma_T}{\mathds{E}[\mathsf{O}_{\phi_j}(T)]}\cdot w_j^{-1}\right) \\
&\leq w_i \cdot \min_{\Phi_H^{(i)}}\left( \frac{16\alpha\log T}{\Delta_{\phi_j}^2 w_j} + \frac{4\alpha w_j^{-1}}{2\alpha-1} \right).
\end{aligned}
\]
This completes the proof for Lemma \ref{lem:Offload-H3}.
\end{proof}}
\begin{proof}[Proof of Theorem \ref{thm:stoch} (a)]
We use Lemma \ref{lem:Offload-H3} to upper-bound the regret due to the offloads corresponding to the confidence measures in $\Phi_H$. Meanwhile, for the confidence measures in $\Phi_L$, with the stochastic arrival process being a special case of the general adversarial arrivals, the following bound follows directly from Lemma \ref{lem:Accept-L} with
\[\mathds{E}\left[\sum_{\phi_i \in \Phi_L}A_{\phi_i}(T)\right]  \leq 2\alpha \left(\frac{|\Phi_L|+1 }{{2\alpha-1}}\right). \]

For stochastic arrivals, we know that
\[ R_{\pi}^S(T)\leq \sum_{\phi_i \in \Phi_H}\Delta_{\phi_i}\, \mathds{E}\left[\mathsf{O}_{\phi_i}(T)\right] +  \max_{\Phi_L}\Delta_{\phi_i}\mathds{E}\left[\sum_{\phi_i \in \Phi_L}A_{\phi_i}(T)\right].\]
Therefore, using Lemma \ref{lem:Offload-H3} and Lemma \ref{lem:Accept-L},  we obtain
\[
\begin{aligned}
&R_{\pi}^S(T) \leq \sum_{\phi_i \in \Phi_H} \; \min_{\phi_j \in\Phi_H^{(i)}} \left\{\frac{16\alpha w_i \Delta_{\phi_i} }{w_j\Delta^2_{\phi_j}}\;\right\}\log T\; +\mathbf{C_3}(\Delta), \text{with}\\
    &\mathbf{C_3}(\Delta)=\frac{2\alpha }{2\alpha-1}\left(\sum_{ \Phi_H}\Delta_{\phi_i}  \min_{\Phi^{(i)}_H}\left(   \frac{ 2w_i}{w_j} \right)+\left(|\Phi_L|+1\right) \max_{ \Phi_L}\; \Delta_{\phi_j} \right).
    \end{aligned}
\]
This completes the proof for Theorem \ref{thm:stoch} (a).
\end{proof}
We  use Lemma \ref{lem:Offload-H3} to provide a regret bound for HI-LCB with stochastic arrivals and constant offloading cost.
\begin{lemma}
\label{lem:Offload-H4}
For any $\alpha>0.5$, the expected number of offloads for samples corresponding to any confidence measure $\phi_i \in \Phi_H$ for HI-LCB with the stochastic arrival sequence provided in Assumption \ref{as:stoch} and known fixed offloading cost can be upper-bounded by
\[ \mathds{E}[\mathsf{O}_{\phi_i}(T)] \leq  w_i \cdot \min_{\phi_j \in\Phi_H^{(i)}}\left( \frac{4\alpha\log T}{\Delta_{\phi_j}^2 w_j} + \frac{2\alpha w_j^{-1}}{2\alpha-1} \right)\cdot \]
\end{lemma}
\begin{proof}
Using Lemma \ref{lem:Offload-H2}, we upper-bound the expected number of offloads for a confidence measure $\phi_i \in \Phi_H$ for HI-LCB with constant offloading cost in the following manner:
\[
\begin{aligned}
\mathds{E}[\mathsf{O}_{\phi_i}(T)] &\leq  w_i\cdot \min_{\phi_j \in\Phi_H^{(i)}}\left( \max_{\sigma_T} {\mathds{E}[ \mathsf{O}_{\phi_j}(T)]}\cdot w_j^{-1}\right) \\
&\leq  w_i \cdot \min_{\phi_j \in\Phi_H^{(i)}}\left( \frac{4\alpha\log T}{\Delta_{\phi_j}^2 w_j} + \frac{2\alpha w_j^{-1}}{2\alpha-1} \right)\cdot
\end{aligned}
\]
This completes the proof of Lemma \ref{lem:Offload-H4}.
\end{proof}
\begin{proof}[Proof of Theorem \ref{thm:stoch} (c)]
Using Lemma \ref{lem:Offload-H4} to upper-bound the regret contribution from $\Phi_H$ and Lemma \ref{lem:Accept-L2} for upper-bounding the regret due to $\Phi_L$, we obtain the cumulative regret bound for $\pi=\text{HI-LCB}$ with stochastic arrivals and known fixed offload cost as
\[
\begin{aligned}
&R_{\pi}^S(T) \leq \sum_{\phi_i \in \Phi_H} \; \min_{\phi_j \in\Phi_H^{(i)}} \left\{\frac{4\alpha w_i \Delta_{\phi_i} }{w_j\Delta^2_{\phi_j}}\;\right\}\log T\; +\mathbf{C_4}(\Delta), \text{with}\\
    &\mathbf{C_4}(\Delta)=\frac{2\alpha }{2\alpha-1}\left(\sum_{ \phi_i \in\Phi_H}\Delta_{\phi_i}  \min_{\Phi^{(i)}_H}\left(   \frac{ w_i}{w_j} \right)+|\Phi_L|\max_{ \Phi_L}\; \Delta_{\phi_j} \right).
    \end{aligned}
\]
This completes the proof of Theorem \ref{thm:stoch} (c).
\end{proof}
\subsection{Proof of Theorem \ref{thm:lb}}
{
To establish the lower bound in Theorem \ref{thm:lb}, we consider a simple instance with a singleton the confidence set $\Phi = \{ \phi_1 \}$, and the offloading cost $\Gamma_t$ drawn i.i.d. from a $\text{Bernoulli}(\gamma)$ distribution with $\gamma > 1 - f(\phi_1)$.

 \begin{proof}[Proof of Theorem \ref{thm:lb}]  Let $\mathsf{O}_{\phi_1}(T)$ denote the number of offloads until time $T$, for the abovementioned setup, the adversarial regret of any policy is given by
 \[R_\pi^A(T) = \Delta_{\phi_1}\cdot \mathds{E}[\mathsf{O}_{\phi_1}(T)]. \]
 
We restrict our attention to policies that guarantee sublinear regret in this setup, i.e., $R_\pi^A(T) = o(T)$, without assuming knowledge of either the mean offloading cost $\gamma$ or the inference accuracy $f(\phi_1)$. 

We define $\mathbb{P}_{\gamma}$ as the true probability measure under which $\mathbb{E}[\Gamma_t] = \gamma$, and $\mathbb{P}_{\gamma'}$ as the probability measure under the alternative hypothesis where $\mathbb{E}[\Gamma_t] = \gamma'\leq 1 - f(\phi_1)$. Let $Y$ denote the event
\[Y :=  \lim_{T\rightarrow\infty} \frac{\mathsf{O}_{\phi_1}(T)}{T} \geq \frac{1}{2}   .\]

Under the true distribution $\mathbb{P}{\gamma}$, any policy achieving a sublinear $o(T)$ regret must ensure that the event $Y$ occurs with a low probability with
\begin{equation}
\label{eq:P-gamma}
\mathbb{P}_{\gamma}(Y) \leq \delta \; \text{for some } \delta \in (0,1).
\end{equation}

On the other hand, under the alternate distribution $\mathbb{P}{\gamma'}$, any such policy guaranteeing a sublinear regret must satisfy
\begin{equation}
\label{eq:P-1-f}
\mathbb{P}_{\gamma'}(Y) \geq 1 - \delta.
\end{equation}

Applying the \textit{Bretagnolle-Huber Inequality} \cite{LAI19854}, we have
\begin{equation}
    \mathbb{P}_{\gamma}(Y) + \mathbb{P}_{\gamma'}(\overline{Y}) \geq \frac{1}{2} \exp\left(-D(\mathbb{P}_{\gamma} \,\|\, \mathbb{P}_{\gamma'})\right),
\end{equation}
where $D(\cdot \,\|\, \cdot)$ denotes the \textit{Kullback–Leibler divergence} between two probability distributions. Simplifying the expression for KL divergence using \eqref{eq:P-gamma} and \eqref{eq:P-1-f}, we obtain
\begin{equation}
\label{eq:kl1}
    D(\mathbb{P}_{\gamma} \,\|\, \mathbb{P}_{\gamma'})  \geq \log\left(\frac{1}{4\delta}\right).
\end{equation}
For the Hierarchical Inference setup given in Section \ref{sec:Terminology}, the only source of information distinguishing $\mathbb{P}_{\gamma}$ from $\mathbb{P}_{\gamma'}$ is the outcome of observed inferences during offloads. Therefore, using the \textit{Diverge Decomposition Lemma} \cite{lattimore2020bandit}, we obtain
\begin{equation}
\label{eq:kl2}
\begin{aligned}
  D(\mathbb{P}_{\gamma} \,\|\, \mathbb{P}_{\gamma'}) &\leq \lim_{T\rightarrow\infty} \mathbb{E}[\mathsf{O}_{\phi_1}(T)] \cdot D_B( \gamma\,\|\,\gamma'),
  \end{aligned}
\end{equation}
 where $D_B(p \,\|\, q)$ denotes the Kullback–Leibler divergence between two Bernoulli distributions with parameters $p$ and $q$. Combining \eqref{eq:kl1} and \eqref{eq:kl2} yields
\[
 \lim_{T\rightarrow\infty} \mathbb{E}[\mathsf{O}_{\phi_1}(T)] \cdot D_B(\gamma\,\|\,\gamma') \geq \log\left(\frac{1}{4\delta}\right),  
\]
for any $\gamma' \leq 1-f(\phi_1)$.

Setting $\delta=1/T$, the cumulative regret incurred by any policy $\pi$ for the given instance with a singleton confidence measure set is lower bounded by
\[
\begin{aligned}
 \min_\pi \lim_{T\rightarrow\infty} R^A_\pi(T) &= \min_\pi \lim_{T\rightarrow\infty} \Delta_{\phi_1}\cdot \mathds{E}[\mathsf{O}_{\phi_1}(T)]\\&\geq  \max_{\{\gamma'\leq 1-f(\phi_1)\}}\,  \frac{\Delta_{\phi_1}\log T}{D_B(\gamma\,\|\,\gamma')}+O(1)\\
 &\geq \frac{\Delta_{\phi_1}\log T}{D_B(\gamma\,\|\,1-f(\phi_1))}+O(1),
 \end{aligned}
\]
i.e., $\min R_\pi^A(T)  = \Omega(\log T)$ for $\Phi=\{\phi_1\}$ with $\gamma>1-f(\phi_1)$, thus proving Theorem \ref{thm:lb}.
\end{proof}}

\section*{Appendix B}

In this section, we consider an alternative cost function where we use the ground truth as the baseline. Let $L_t^{\pi}$ denote the cost incurred by the system at time $t$ under policy $\pi$. It follows that 
\begin{equation*}
\label{eq:cost_with_ground_truth}
    L_t^{\pi} = (\Gamma_t + \eta(h_r,x_t)) \cdot D_\pi(t) + \eta(h_l,x_t) \cdot (1-D_\pi(t)).
\end{equation*}

\begin{figure}[h]
    \begin{subfigure}{\linewidth}
        \includegraphics[width=0.90\linewidth]{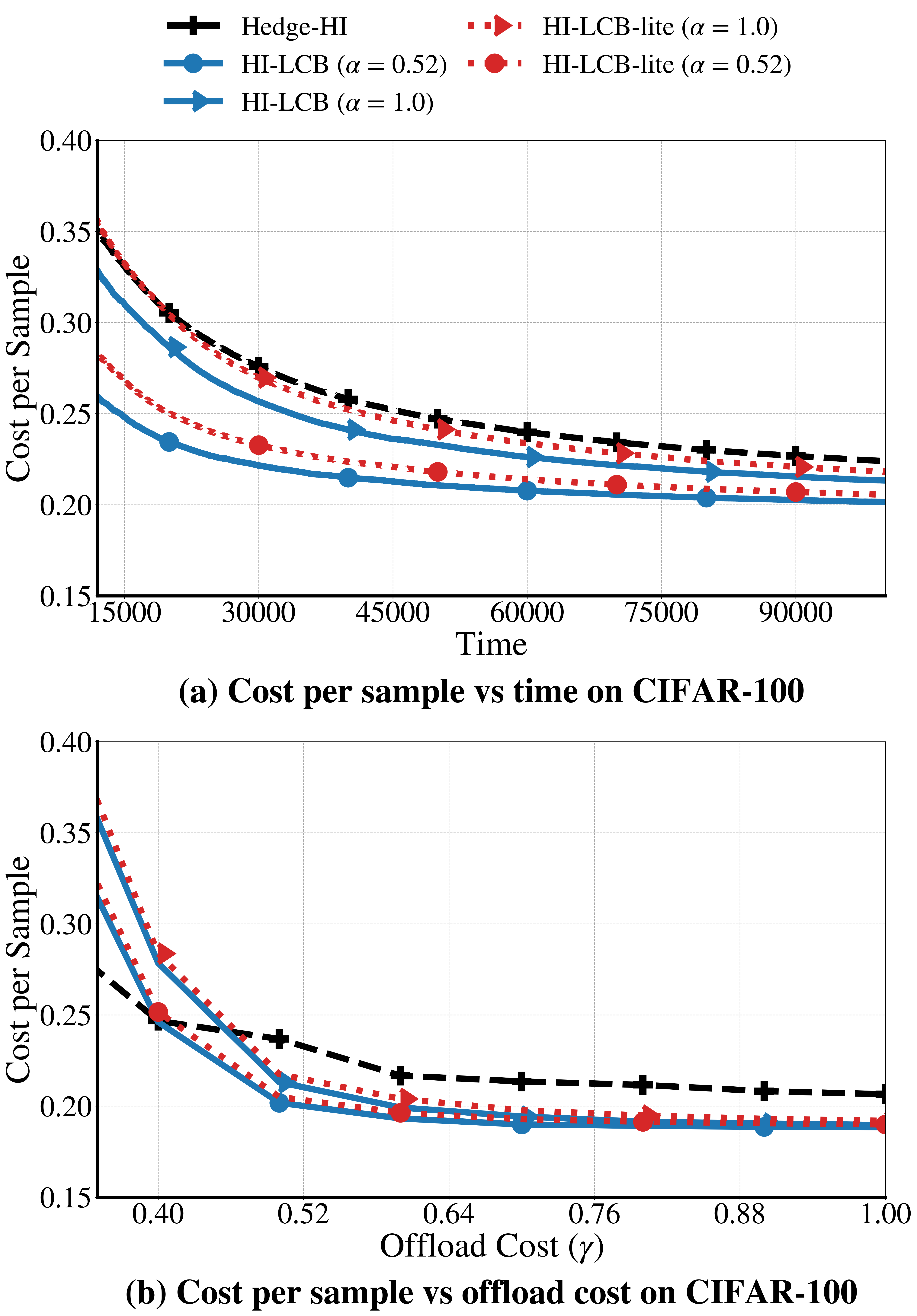}
    \end{subfigure}
    \caption{Simulation results for CIFAR-100 using the ground-truth as the baseline.}
    \label{fig:cifar100}
\end{figure}

Note that the information available to the policy is identical to that in Section \ref{sec:Terminology}, i.e., the policy observes the offloading cost and the inference of the Remote-ML model only when it offloads and the policy never observes the ground truth.

We evaluate the performance of various candidate policies in terms of the average cost per sample, defined as $(\sum_{\tau = 1}^{t} L_{\tau}^{\pi})/t$, as a function of different parameters. We compare the performance of HI-LCB and HI-LCB-lite alongside HI-Hedge for image classification on the CIFAR-100 dataset with horizon $T = 100\,000$, using ResNet-50 (53\% accuracy) as the Local-ML model and WideResNet-28-10 (82\% accuracy) as the Remote-ML model.

\begin{remark}
The static policy from Lemma~\ref{lem:opt} does not serve as an appropriate baseline here because it uses the Remote-ML inference as a proxy for the ground-truth. We therefore compare the average cost per sample of the various policies in place of comparing their regret.
\end{remark}

In Figure~\ref{fig:cifar100}(a), we illustrate the evolution of per-sample inference cost for a system with a known fixed offloading cost of $\gamma = 0.5$. Figure~\ref{fig:cifar100}(b) shows the per-sample inference cost at $T = 100\,000$ as a function of the offloading cost $\gamma$. Finally, in Figure~\ref{fig:cifar100-2}, we examine how the per-sample inference costs of HI-LCB and HI-LCB-lite vary with the exploration parameter $\alpha$, for a fixed offloading cost of $\gamma = 0.5$ and  $T = 100\,000$.

The results for this setup closely mirror those in Section~\ref{sec:results}, with HI-LCB and HI-LCB-lite outperforming HI-Hedge. For $T = 100\,000$, offloading cost $\gamma = 0.5$, and $\alpha = 0.52$, HI-LCB (80.5\% accuracy) and HI-LCB-lite (80.4\% accuracy) achieve more accurate inference than HI-Hedge (79.1\% accuracy), while incurring lower per-sample costs on CIFAR-100.

\begin{figure}[t]
    \centering
    \begin{subfigure}{\linewidth}
        \centering
        \includegraphics[width=0.90\linewidth]{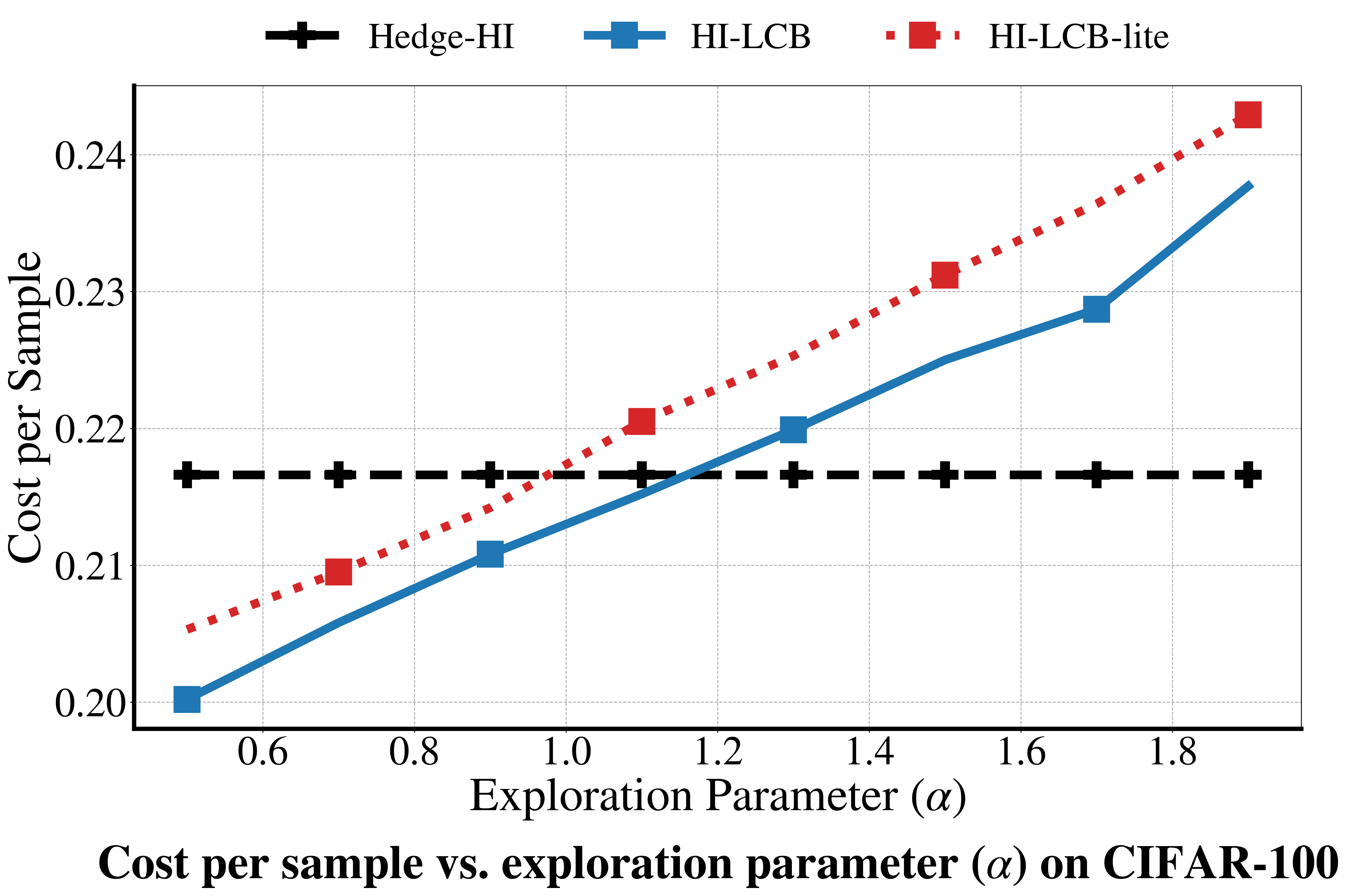}
    \end{subfigure}

    \caption{Variation of per-sample inference cost with exploration parameter ($\alpha$) for HI-LCB and HI-LCB-lite for CIFAR-100 using the ground-truth as the baseline.}
    \label{fig:cifar100-2}
\end{figure}

\end{document}